\documentclass{article}  
\PassOptionsToPackage{numbers,sort&compress}{natbib} 
\usepackage[main,final]{neurips_2026}
\setcitestyle{numbers,square,comma,sort&compress}
\usepackage{comment}

\usepackage[utf8]{inputenc}  
\usepackage[T1]{fontenc}  
\usepackage[hypertexnames=false]{hyperref}  
\usepackage{url}  
\usepackage{booktabs}  
\usepackage{amsfonts,amsmath,amssymb,amsthm,mathtools,bm}  
\usepackage{nicefrac}  
\usepackage{microtype}  
\usepackage{xcolor}  
\usepackage{enumitem}  
\usepackage{algorithm}  
\usepackage{algpseudocode}  
\usepackage{graphicx}  
\usepackage{longtable}  
\usepackage{array}  
\usepackage{tikz}  
\usepackage{pgfplots}  
\usepackage{pgfplotstable}  
\usepackage{multirow}  
\usepackage{tabularx}  
\usepackage{float}
\usepackage{threeparttable}  
\usepackage{makecell}  
\pgfplotsset{compat=1.17}  
\usepgfplotslibrary{fillbetween,groupplots}  

\hypersetup{  
	colorlinks=true,  
	linkcolor=blue,  
	citecolor=blue,  
	urlcolor=blue  
}  

\makeatletter  
\renewcommand{\theHfigure}{\arabic{figure}}  
\renewcommand{\theHtable}{\arabic{table}}  
\renewcommand{\theHalgorithm}{\arabic{algorithm}}  
\makeatother

\definecolor{calblue}{RGB}{33,102,172}  
\definecolor{syncgreen}{RGB}{31,158,137}  
\definecolor{partorange}{RGB}{230,97,1}  
\definecolor{badred}{RGB}{178,24,43}  
\definecolor{pcapurple}{RGB}{94,60,153}  
\definecolor{softgray}{RGB}{90,90,90}  

\newtheorem{theorem}{Theorem}  
\newtheorem{proposition}[theorem]{Proposition}  
\newtheorem{lemma}[theorem]{Lemma}  
\newtheorem{corollary}[theorem]{Corollary}  
\newtheorem{definition}[theorem]{Definition}  
  
\newtheorem{remark}[theorem]{Remark}  
\newtheorem{conjecture}[theorem]{Conjecture}  

\newcommand{\E}{\mathbb E}  
\newcommand{\Pbb}{\mathbb P}  
\newcommand{\R}{\mathbb R}  
\newcommand{\cV}{\mathcal V}  
  
\newcommand{\KL}{\mathrm{KL}}  
\newcommand{\MI}{\mathrm I}  
  
\newcommand{\Corr}{\mathrm{Corr}}  
  
\newcommand{\Unif}{\mathrm{Unif}}  
\newcommand{\Ber}{\mathrm{Ber}}  
\newcommand{\Ksp}{K_{\mathrm{sp}}}  
\newcommand{\Nsrc}{N_{\mathrm{src}}}  
\newcommand{\epsrec}{\varepsilon_{\mathrm{rec}}}  
\newcommand{\Supp}{\mathrm{supp}}  
\newcommand{\sgn}{\mathrm{sign}}  
\newcommand{\norm}[1]{\left\lVert #1\right\rVert}  
\newcommand{\ip}[2]{\left\langle #1,#2\right\rangle}  

\setlist[itemize]{leftmargin=1.2em,itemsep=0.05em,topsep=0.15em}  
\setlist[enumerate]{leftmargin=1.2em,itemsep=0.05em,topsep=0.15em}  

\newcommand{\ind}{\mathbf{1}}

\newcommand{\GSL}{\textsc{GSL}}

\newcommand{\AlignedMass}{\emph{aligned mass}}  
  
\newcommand{\Labels}{\textsc{Labels}}  
\newcommand{\Calibration}{\textsc{Calibration}}

\title{  
	Can Metadata Fix the Gauge?\\  
	Calibration Turns Sparse Multi-Source Learning\\  
	from Sparse PCA into Sparse Mean Recovery  
}

\author{%
		Yirong Xiang\\
	Air Force Medical Center\\
	\texttt{yirong\_xiang@163.com}
	\And
	Yibo Zhou\thanks{Corresponding author.}\\
Beihang University\\
\texttt{ybzhou@buaa.edu.cn}
}

\begin{document}  
	\maketitle  
\begin{abstract}    
	Source labels tell us where a sample came from, but not how that source oriented its coordinates. In a sparse Gaussian multi-source model with source-specific sign gauges, that distinction changes the learning problem. Without calibration, pooling across sources erases the mean signal, and the source-labelled experiment reduces to sparse PCA. With task-independent relative-gauge metadata, synchronization aligns sources and the same observations become a sparse mean problem. We formalize the value of metadata through \AlignedMass{}, which quantifies how much first-order signal a metadata-induced alignment restores. Source labels create only limited \AlignedMass{}, whereas valid calibration can create \AlignedMass{} proportional to the full sample size. This yields a statistical separation in micro-source regimes and, under standard sparse-PCA hardness assumptions, a computational separation as well. Synthetic experiments check the predicted scaling; controlled real-feature
	and physical sensor studies examine sign cancellation, separately
	acquired calibration, and reuse across tasks.
\end{abstract}  
	
	\vspace{-0.5em}  
	
	\section{Introduction}  
	
Multi-source learning is written as if every source uses the same coordinate convention. In practice this convention can fail. A sensor polarity may be reversed, a scanner pipeline may reflect a latent factor, or a client representation may flip the sign of a hidden component. Within each source the  
signal may be coherent, but across sources naive pooling can cancel it. This paper asks when metadata can fix such a source-specific gauge, and why this is different from merely knowing source labels.  

A concrete example is a sensor whose mounting reverses the sign of a measured
axis. Source identity tells us which recordings share a mounting, but not
whether two mountings use the same or opposite axis convention. A separately
recorded reference motion can provide that comparison before any downstream
task is observed. In our physical sensor study
(Section~\ref{sec:physical-calibration}), these same/opposite records are
synchronized once and reused across four classification tasks. Here
\emph{synchronization} means estimating consistent source orientations up to
one global sign; \emph{linearization} means averaging after this alignment
to restore a mean signal. We study this question in a stylized sparse Gaussian model. Each source observes the  same unknown sparse direction, but through an unknown source-specific sign gauge.  
Source labels tell the learner which samples belong together, yet they do not reveal how one source is oriented relative to another. Calibration metadata is different: relative polarity logs, orientation edges, or noisy absolute orientation records can carry information about the source gauges. Such metadata can be synchronized into a common convention, using ideas related to synchronization over \(\mathbb Z_2\) and other groups \citep{singer2011angular,bandeira2013cheeger}. The main message is that these two kinds of side information induce different statistical experiments. With only source labels, the learner can aggregate within each source, but the unknown gauges remove first-order signal across sources. The resulting source-labelled experiment is exactly sparse PCA, plus residual Gaussian noise independent of the task direction. With valid  
	relative-gauge metadata, synchronization selects a gauge section and restores  
	first-order structure. Conditional on source amplitudes, the same downstream  
	observations become a sparse mean problem. Thus labels identify groups, whereas  
	calibration identifies orientation.  
	
Source identifiers, group labels, and institution tags are useful for organizing data and modeling heterogeneity  
\citep{bendavid2010theory,kairouz2021federated,arjovsky2019invariant,sagawa2020distributionally,koh2021wilds,zhou2025heterogeneous},but they do not by themselves specify how source-specific coordinate conventions should be aligned.  
Calibration metadata plays a narrower role: it provides information about nuisance source conventions rather than the downstream task.   
In our model, this distinction changes the statistical experiment: the available signal changes from second-order to first-order, leading to different statistical and computational thresholds.  
	
	To quantify the value of metadata, we introduce \AlignedMass{}. Given metadata-induced sample weights \(q\) and the hidden gauge vector
	\(g\), \AlignedMass{} is \(n\) times their squared cosine similarity.
	In the constant-amplitude companion model, \(\theta^2M_q(g)\) is
	exactly the squared first-order signal strength. In the source-Gaussian model it still controls the nominal first-order  
	signal, while an aligned-source-count condition controls concentration of random  
	source amplitudes. This viewpoint makes the comparison sharp: source labels can  
	create only limited \AlignedMass{} because they know which samples share a source  
	but not which way that source points; valid relative-gauge metadata can create  
	\AlignedMass{} proportional to the full sample size.  
	
	The resulting separation is most transparent in the micro-source regime, where  
	each source is individually weak. Calibrated recovery succeeds at the sparse-mean  
	scale, while source-labelled uncalibrated recovery is statistically insufficient  
	at the corresponding sparse-PCA scale. Under standard sparse-PCA hardness  
	assumptions  
	\citep{berthet2013complexity,deshpande2014sparse}, there is also a conditional  
	computational separation: calibrated recovery is polynomial-time after  
	synchronization, whereas source-labelled recovery remains as hard as sparse PCA  
	through the usual hard window.  
	
	A useful feature of calibration metadata is that it is task-independent. It  
	fixes a source convention rather than revealing the downstream sparse direction.  
	One relative-gauge graph can therefore be collected and synchronized once, then  
	reused across many independent downstream tasks sharing the same source gauges.  
	This differs from feature- or label-derived alignment, which may be useful but  
	uses task-coupled information rather than external calibration. Our contributions are:  
	
	\begin{enumerate}[leftmargin=1.3em,itemsep=0.15em,topsep=0.15em]  
		\item We identify a sparse multi-source setting in which source-labelled  
		uncalibrated data is statistically equivalent to sparse PCA, while calibrated  
		data becomes sparse mean recovery.  
		\item We introduce \AlignedMass{} as a quantitative measure of the first-order  
		value of metadata-induced alignment, and show that source labels cannot  
		reliably create extensive \AlignedMass{} whereas valid relative-gauge metadata  
		can.  
		\item We derive an unconditional statistical separation in micro-source  regimes and a conditional computational separation under standard sparse-PCA  hardness. Experiments check phase transitions, uninformative or mismatched
		metadata controls, calibration reuse, and controlled real-feature behavior.
	\end{enumerate}  

Spectral synchronization, weighted averaging, and top-\(k\) thresholding are
standard components. Our contribution is the experiment-level converse
without calibration, together with the quantitative link from noisy
calibration to restored first-order signal, rather than a new
synchronization algorithm.
	
	\section{Model, metadata, and alignment}  
	\label{sec:model}  
	
	\subsection{A sparse multi-source model}  
	
	Let  $
	\cV_{d,k}  
	=  
	\{v\in\{0,\pm 1/\sqrt{k}\}^d:\norm{v}_0=k\},  
	\Ksp=k\log(ed/k).  
	$ There are \(L\) sources, each contributing \(s\) samples, so the total sample size is \(n=Ls\). The unknown sparse direction is \(V\in\cV_{d,k}\). Source \(\ell\) has an unknown sign gauge \(\sigma_\ell\stackrel{\mathrm{i.i.d.}}{\sim}\Unif\{\pm1\}\) and a nonnegative amplitude \(R_\ell\stackrel{\mathrm{i.i.d.}}{\sim}|N(0,1)|\), independent of \(\sigma_\ell\). Writing \(a_\ell=\sigma_\ell R_\ell\sim N(0,1)\), the observation model is  
	\[  
	W_{\ell r}  
	=  
	\theta R_\ell \sigma_\ell V+\varepsilon_{\ell r},  
	\qquad  
	\varepsilon_{\ell r}\sim N(0,I_d),  
	\qquad  
	\ell=1,\dots,L,\ r=1,\dots,s.  
	\]  
	Here \(\theta>0\) is the signal magnitude,
	\(\norm{v}_0\) counts nonzero coordinates, and
	\(\Ksp=k\log(ed/k)\) is the sparse-complexity term, with \(e\) Euler's number
	and \(\log\) the natural logarithm. The feature noises are mutually
	independent and independent of \(V,\sigma,R\).
	The Rademacher gauge prior means that each source sign is independently
	and uniformly sampled from \(\{-1,+1\}\).
	The downstream theoretical task is to recover \(V\) up to a global sign,
	or its support \(\Supp(V)=\{j:V_j\neq0\}\).
	We write \(\gtrsim\), \(\lesssim\), and \(\asymp\) for inequalities or
	equivalence up to universal positive constants. The same statistic arises from a simple two-view classification model; the derivation appears in Appendix~\ref{app:two-view}.  The model separates two roles that are easy to conflate:  
	
	\begin{itemize}[leftmargin=1.3em,itemsep=0.15em]  
		\item \(\sigma_\ell\) is a \emph{coordinate convention}. This is the part calibration metadata may reveal.  
		\item \(R_\ell\) is a \emph{source amplitude}. This is nuisance randomness and is not something valid metadata should depend on.  
	\end{itemize}  
	
Probabilities involving \(\sigma\) are with respect to this Rademacher prior. We also use the constant-amplitude companion \(W_{\ell r}=\theta\sigma_\ell V+\varepsilon_{\ell r}\), where \AlignedMass{} is exact.  

\subsection{What counts as valid calibration metadata?}

Source labels identify which samples belong to the same source.
Calibration metadata instead provides information about source signs.
A relative-polarity log or orientation edge is a noisy record of whether
two sources use the same or opposite sign convention; a noisy absolute
orientation record compares one source with a known reference.

Let \(E\) contain only exogenous source-side structure, such as source
identities, source sizes, calibration-graph topology, and a pre-specified
calibration design. In the theoretical model, \(E\) is fixed or sampled
independently of \(\sigma,R,V\) and task randomness; it cannot encode
downstream labels or features. Let \(D_{\rm task}\) collect downstream
observations, their feature noise, and any downstream labels. In the
two-view example, these labels are \(Y_{\ell r}\); the main recovery model
requires no additional labels.

\begin{definition}[Valid calibration metadata]
	A metadata channel \(\mathfrak M\) is \emph{valid calibration metadata} if
	\[
	P(\mathfrak M\mid V,R,D_{\rm task},\sigma,E)
	=
	P(\mathfrak M\mid \sigma,E).
	\]
	Thus, conditional on source signs and the exogenous source design,
	calibration is independent of the downstream direction, amplitudes,
	observations, labels, and feature noise.
\end{definition}

Validity concerns task independence, not accuracy. Random metadata can be
valid but uninformative; missing or independently corrupted calibration
records can also remain valid. Conversely, feature- or label-derived
alignment can be useful while violating task independence.

With externally collected valid calibration, weights are estimated without
using any downstream observations, so scarce micro-source samples need
not be split to estimate alignment. Sample splitting is relevant to
task-derived comparison procedures: it can separate weight estimation
from the noise in a final statistic, but does not make that alignment
task-independent or reusable calibration.

	\subsection{\AlignedMass{}: how much first-order signal metadata restores}
	
	Let \(q\in\R^n\) be nonzero sample weights constructed from calibration
	metadata and source structure. For example, estimated source signs
	\(\widehat\sigma_\ell\) give \(q_{\ell r}=\widehat\sigma_\ell\);
	spectral synchronization can instead supply real-valued source weights.
	We also call such metadata-derived weights a \emph{section}.
	At the sample level, the hidden sign vector is
	\(g_{\ell r}=\sigma_\ell\).
	
	The constant-amplitude companion model makes the role of these weights
	explicit:
	\[
	X_q
	=
	\frac{\sum_{\ell,r}q_{\ell r}W_{\ell r}}{\norm q_2}
	=
	\theta\frac{q^\top g}{\norm q_2}V+Z,
	\qquad Z\sim N(0,I_d).
	\]
	This motivates
	\[
	M_q(g)=\frac{(q^\top g)^2}{\norm q_2^2},
	\]
	because \(\theta^2M_q(g)\) is exactly the squared restored mean-shift
	strength. Since \(\norm g_2^2=n\),
	\[
	0\le M_q(g)\le n,
	\qquad
	\frac{M_q(g)}{n}
	=
	\frac{(q^\top g)^2}{\norm q_2^2\norm g_2^2}
	=
	\cos^2(q,g).
	\]
	Thus \AlignedMass{} is an effective first-order sample size, and
	\(M_q(g)/n\) is normalized alignment. No assumption
	\(\norm q_2=\norm g_2\) is needed. Exact aligned mass is a theoretical and
	evaluation quantity, not an algorithm input: the algorithm uses \(q\),
	not the unknown \(g\). Planted signs are available for evaluation in
	simulations; recorded mounting signs in the physical study are used only
	for alignment evaluation and the oracle control.
	
	For equal source size \(s\), we can also quantify the metadata channel's
	posterior expected alignment. Define the posterior source-sign correlation
	matrix
	\[
	\Gamma(\mathfrak M,E)
	=
	\E[\sigma\sigma^\top\mid\mathfrak M,E]
	\in\R^{L\times L},
	\]
	and its alignment capacity
	\[
	\mathcal C(\mathfrak M,E)
	=
	\sup_{q=q(\mathfrak M,E)\neq0}
	\E[M_q(g)\mid\mathfrak M,E].
	\]
	The supremum is over deterministic metadata-measurable weights;
	independent algorithmic randomness may be conditioned on.
	
	\begin{proposition}[Posterior alignment capacity]
		\label{prop:alignment-capacity}
		For equal source size \(s\),
		\[
		\mathcal C(\mathfrak M,E)
		=
		s\,\lambda_{\max}\!\left(\Gamma(\mathfrak M,E)\right),
		\]
		where \(\lambda_{\max}\) denotes the largest eigenvalue.
		An optimal source-level weighting is a top eigenvector of
		\(\Gamma(\mathfrak M,E)\).
	\end{proposition}
	
	Under the independent Rademacher prior, source labels give
	\(\Gamma=I_L\), whereas perfect relative calibration determines
	\(\sigma\sigma^\top\), even though the global sign remains unknown.
	Consequently,
	\[
	\text{\Labels}:\ \mathcal C=s
	\qquad\qquad
	\text{\Calibration}:\ \mathcal C=n.
	\]
	This is a posterior expectation statement; the synchronization result
	below controls realized alignment with high probability.
	
	In the source-Gaussian model, amplitudes vary across sources. Write
	\[
	Q_\ell(q)=\sum_{r=1}^s q_{\ell r},
	\]
	and define the aligned source count
	\[
	\Nsrc(q,g)
	=
	\frac{\left(\sum_{\ell=1}^LQ_\ell(q)\sigma_\ell\right)^2}
	{\sum_{\ell=1}^LQ_\ell(q)^2},
	\]
	with value zero when the denominator vanishes.
	\AlignedMass{} controls the nominal first-order signal, while
	\(\Nsrc\) controls concentration of the weighted random amplitudes around
	their mean. High aligned mass concentrated on too few sources need not
	produce a stable realized signal.
	
	\section{Main results}  
	\label{sec:theory}  
	
	We first show that source labels alone expose only second-order information. We then show how calibration metadata restores first-order structure.  
	
	\subsection{Source labels alone reduce to sparse PCA}  
	
If source labels are available but gauges are not aligned, the sufficient source  
statistic is  
\(X_\ell=s^{-1/2}\sum_r W_{\ell r}=\theta\sqrt{s}\,a_\ell V+Z_\ell\), with  
\(Z_\ell\sim N(0,I_d)\). Since \(a_\ell\sim N(0,1)\),  
\(X_\ell\sim N(0,I_d+\lambda VV^\top)\), where \(\lambda=\theta^2s\).

	\begin{proposition}[Source-labelled data is sparse PCA]  
		\label{prop:source-spca}  
In the source-Gaussian model with equal source size \(s\) under the Rademacher gauge prior, the source-labelled uncalibrated observations are statistically equivalent to   
		\[  
		X_\ell\sim N(0,I_d+\lambda VV^\top),  
		\qquad  
		\ell=1,\ldots,L,  
		\qquad  
		\lambda=\theta^2s,  
		\]  
together with residual Gaussian noise independent of \(V\). Hence any source-labelled uncalibrated estimator can be simulated from \(L\) sparse PCA samples plus \(V\)-independent randomness. 
	\end{proposition}  
	
	Under the Rademacher-gauge experiment, this is an equivalence of statistical  
	experiments, not a statement about a particular algorithm or a particular  
	linear statistic.  Consequently it already covers arbitrary feature-dependent  
	procedures applied to the source-labelled uncalibrated samples, including  
	attempts to infer source signs or align sources from the downstream features  
	themselves.  Such procedures can only use the sparse-PCA information in the  
	source aggregates \(X_\ell\) together with \(V\)-independent residual noise.  
	In this sense, feature-derived self-alignment from the uncalibrated task data  
	is not external calibration metadata; if it succeeds in the corresponding hard  
	regime, it is exploiting or solving the sparse-PCA instance exposed by  
	Proposition~\ref{prop:source-spca}.

Thus marginalizing the gauges removes first-order signal and leaves only the second-order spike.  
	
	\subsection{Calibration turns the same observations into a sparse mean problem}  
	
	Now suppose metadata supplies a section \(q\). Form the weighted linear statistic  
	\[  
	X_q=\frac{\sum_{\ell,r}q_{\ell r}W_{\ell r}}{\norm{q}_2}.  
	\]  
	Conditional on \(q,\sigma,R,V\), we obtain  
	\[  
	X_q=\theta A_qV+Z,  
	\qquad  
	Z\sim N(0,I_d),  
	\qquad  
	A_q=\frac{\sum_{\ell,r}q_{\ell r}R_\ell\sigma_\ell}{\norm{q}_2}.  
	\]  
	So after gauge selection, the problem is sparse mean recovery.  
	
	In the constant-amplitude model, \(A_q^2=M_q(g)\) exactly. In the source-Gaussian model, \AlignedMass{} still governs the nominal signal and aligned source count guarantees concentration.  
	
\begin{theorem}[Calibration linearizes sparse recovery]
	\label{thm:linearizes}
	Let \(q\neq0\) be constructed exclusively from valid calibration metadata
	and \(E\), with independent algorithmic randomness, and let
	\(0<\epsrec,\delta<1/2\).
	Conditional on \(q,\sigma,E,V\), suppose
	\[
	\theta^2\kappa_1^2M_q(g)
	\gtrsim
	\epsrec^{-2}\bigl(\Ksp+\log(2/\delta)\bigr),
	\qquad
	\kappa_1=\sqrt{2/\pi},
	\]
	and
	\[
	\Nsrc(q,g)\gtrsim\log(2/\delta).
	\]
	Then top-\(k\) thresholding of \(X_q\), with observed coordinate signs,
	returns \(\widehat V\) satisfying
	\[
	|\ip{\widehat V}{V}|\ge1-\epsrec
	\]
	with probability at least \(1-\delta\) over the source amplitudes and
	downstream feature noise, up to universal constants.
\end{theorem}
	
	The proof combines a standard sparse-mean thresholding bound with a weighted half-normal concentration lemma; both are given in Appendix~\ref{app:proofs}.  
	
	\subsection{Why source labels cannot create enough \AlignedMass{}}  
	
	The next lemma formalizes a simple but decisive limitation: if a weighting rule only knows source labels and not source gauges, then its correlation with the hidden gauge vector behaves like a Rademacher sum.  
	
	\begin{lemma}[Source labels do not choose a gauge]  
		\label{lem:no-section-main}  
	Assume equal source size \(s\) and the Rademacher gauge prior. Let \(q\) be any possibly randomized section measurable with respect to source labels and independent randomness, but independent of the unknown gauges \(\sigma\). Then, for every \(0<\delta<1\),  
		\[  
		\Pbb_\sigma\left(  
		M_q(g)\le 2s\log\frac{2}{\delta}  
		\ \middle|\ q  
		\right)\ge 1-\delta .  
		\]  
Thus source labels alone cannot reliably create \AlignedMass{} \(M_q(g)=\Theta(n)\) unless \(s=\Theta(n)\).  
	\end{lemma}  
	
Thus labels group samples but do not choose their orientation.

	\subsection{Relative-gauge metadata creates extensive \AlignedMass{}}  
	
	A natural calibration channel is a noisy relative-sign graph over sources. For each \(1\le \ell<m\le L\), let \(A_{\ell m}\sim\Ber(p)\) independently, where \(\Ber(p)\) is the
	Bernoulli distribution with success probability \(p\), with \(A_{m\ell}=A_{\ell m}\) and \(A_{\ell\ell}=0\). On observed edges \(A_{\ell m}=1\), we observe  
	\[  
	C_{\ell m}=\sigma_\ell\sigma_m\nu_{\ell m},  
	\qquad  
	\Pbb(\nu_{\ell m}=1)=1-\epsilon,  
	\qquad  
	\alpha=1-2\epsilon>0.  
	\]  
	Set \(C_{m\ell}=C_{\ell m}\), and define the signed adjacency matrix  
	\[  
	B_{\ell m}=A_{\ell m}C_{\ell m}\quad(\ell\neq m),  
	\qquad  
	B_{\ell\ell}=0.  
	\]  
	Then  
	\[  
	\E[B\mid\sigma]=p\alpha(\sigma\sigma^\top-I).  
	\]  
Let \(\widehat u\) be a unit-norm top eigenvector of \(B\). The edge noises \(\nu_{\ell m}\) are independent across unordered source
	pairs and independent of the graph indicators, source variables, and
	downstream feature noise.
	
	\begin{lemma}[Spectral synchronization]  
		\label{lem:spectral-sync-main}  
		There exist universal constants \(C,c>0\) such that if  
		\[  
		pL\ge C\log(L/\delta),  
		\qquad  
		\alpha^2pL\ge C,  
		\]  
		then with probability at least \(1-\delta\),  
		\[  
		\beta^2  
		:=  
		\left|  
		\left\langle  
		\widehat u,\frac{\sigma}{\sqrt L}  
		\right\rangle  
		\right|^2  
		\ge  
		1-\frac{C}{\alpha^2pL}.  
		\]  
	\end{lemma}  
	
	Assign each sample in source \(\ell\) the weight \(q_{\ell r}=\widehat u_\ell\). Then  
	\[  
	M_q(g)=\beta^2n,  
	\qquad  
	\Nsrc(q,g)=\beta^2L.  
	\]  
So synchronization lifts the effective sample size from \(O(s)\), which is the best source-label sections can do, to \(\Theta(n)\) whenever the synchronized section has constant overlap with the hidden gauges.

	\begin{corollary}[Calibrated recovery scale]  
		\label{cor:calibrated-scale-main}  
		For fixed recovery tolerance \(0<\epsrec<1/2\) and failure probability
		\(0<\delta<1/2\), if synchronization yields \(\beta^2=\Omega(1)\) and \(\beta^2L\gtrsim\log(1/\delta)\), then calibrated top-\(k\) recovery succeeds once  
		\[  
		\theta^2\beta^2 n\gtrsim k\log(ed/k),  
		\]  
		up to universal constants. Under the two-view parametrization \(\theta^2=2\mu^2\), this is equivalently  
		\[  
		\mu^2\beta^2 n\gtrsim k\log(ed/k).  
		\]  
	\end{corollary}  
	
	\paragraph{Metadata budget.}  
	The calibrated and uncalibrated experiments use the same downstream feature  
	sample size \(n=Ls\), but calibration also uses a task-independent metadata  
	resource.  In the relative-sign graph model the expected number of calibration  
	edges is  
	\[  
m  
=  
\E|\{(\ell,m):A_{\ell m}=1,\ell<m\}|  
=  
p\binom{L}{2}.  
	\]  
	Thus the elementary spectral guarantee above uses  
	\[  
	m  
	=  
	O\!\left(  
	L\max\{\log(L/\delta),\alpha^{-2}\}  
	\right)  
	\]  
	calibration edges to reach the overlap level certified by  
	Lemma~\ref{lem:spectral-sync-main}, up to universal constants, when \(p\) is  
	chosen at the stated threshold. To certify \(\beta^2\ge1-\varepsilon\), the same lemma gives the
	sufficient additional degree condition
	\[
	\alpha^2pL\gtrsim\varepsilon^{-1},
	\]
	together with the stated logarithmic degree requirement. This cost is separate from downstream feature  
	samples and can be paid once and reused across tasks sharing the same source  
	gauges. The separation therefore compares two metadata channels: source labels have alignment capacity \(s\), whereas relative-gauge metadata, at the stated calibration-edge budget, can produce realized \AlignedMass{} \(\Theta(n)\).  
	
	\subsection{A separation in the micro-source regime}  
	
	The calibration result above follows the sparse-mean scale. By contrast, Proposition~\ref{prop:source-spca} shows that source labels alone give a sparse PCA problem with  
	\[  
	N=L=\frac{n}{s},  
	\qquad  
	\lambda=\theta^2s.  
	\]  
	Its information scale is therefore  
	\[  
	N\lambda^2=\theta^4ns\asymp \mu^4ns.  
	\]  
	
	The following lower bound turns that heuristic into a support-recovery statement.  
	
\begin{theorem}[High-F1 lower bound for source labels]
	\label{thm:source-fano}
	Assume \(1\le k\le d/2\) and \(\Ksp\ge C_0\) for a sufficiently large
	universal constant \(C_0\).
	In the source-Gaussian model, let \(V\) be uniform on the sparse support
	packing constructed in Appendix~\ref{app:proofs}. Define
	\[
	I_{\rm src}
	=
	\frac{L}{2}\bigl[\lambda-\log(1+\lambda)\bigr],
	\qquad
	\lambda=\theta^2s.
	\]
	There exist universal constants \(c,\tau>0\) such that if
	\(I_{\rm src}\le c\Ksp\), then every estimator using source labels alone
	satisfies
	\[
	\Pbb\bigl(
	\mathrm{SupportF1}(\widehat V,V)\ge1-\tau
	\bigr)\le\frac12.
	\]
	In the weak-source regime \(\lambda\le1\), this condition becomes
	\(\mu^4ns\lesssim\Ksp\).
\end{theorem}
	
	\begin{remark}[Bayes prior and fixed gauges]  
		\label{rem:bayes-minimax}  
		The Rademacher prior on \(\sigma\) is used to obtain a hard average-case gauge  
		ensemble.  The lower bound also implies a worst-case fixed-gauge statement:  
		if the success probability averaged over \(\sigma\) is at most \(1/2\), then  
		there exists a fixed gauge vector \(\sigma^\star\) for which the same estimator  
		has success probability at most \(1/2\).  Equivalently, the Bayes lower bound  
		under the Rademacher gauge prior implies a minimax lower bound over fixed  
		unknown gauges.  We do not claim that every fixed gauge is equally hard; for  
		example, if all sources happen to be already aligned, first-order pooling is  
		possible.  The lower bound concerns the absence of calibration in the  
		random-gauge average-case and fixed-gauge minimax senses.  
	\end{remark}

	Together, Theorem~\ref{thm:linearizes} and Theorem~\ref{thm:source-fano} give an unconditional statistical separation. In the micro-source regime \(s\mu^2\ll1\), there is a window where  
	\[  
	\mu^2 n\gtrsim \Ksp  
	\qquad\text{but}\qquad  
	\mu^4ns\lesssim \Ksp.  
	\]  
	In this window, calibrated recovery succeeds while source-labelled recovery is statistically insufficient for high-F1 support recovery.  
	
	\subsection{A conditional computational separation}  
	
	The statistical gap above does not rely on computational assumptions. A computational separation appears when the source-labelled experiment is interpreted through standard sparse-PCA hardness.  
	
\begin{conjecture}[Sparse PCA constant-advantage hardness, informal]  
		\label{conj:spca-main}  
		Consider the testing problem  
		\[  
		N(0,I_d)^{\otimes N}  
		\quad\text{versus}\quad  
		N(0,I_d+\lambda VV^\top)^{\otimes N},  
		\qquad  
		V\in\cV_{d,k}.  
		\]  
In the standard sparse-PCA hard regime, no polynomial-time algorithm distinguishes the two hypotheses with constant advantage throughout  
\(\Ksp\ll N\lambda^2\ll k^2\), up to logarithmic and standard polynomial-scaling qualifications. High-overlap recovery implies detection by sample splitting and a held-out variance test.  
	\end{conjecture}  
	
	\begin{corollary}[Metadata--computation separation]  
		\label{cor:conditional-separation-main}  
		Assume Conjecture~\ref{conj:spca-main}. If \(s\mu^2\ll1\), \(\Ksp\ll k^2\), and polynomial-time synchronization of valid calibration metadata
		yields \(\beta^2=\Omega(1)\), then at  
		\[  
		n_{\rm cal}\asymp \Ksp/\mu^2  
		\]  
		calibration enables polynomial-time high-overlap recovery, while source labels alone are statistically insufficient. Moreover, for larger uncalibrated sample size \(n_{\rm mix}\), source-labelled recovery remains computationally hard throughout the sparse-PCA window  
		\[  
		\frac{\Ksp}{s\mu^4}  
		\ll  
		n_{\rm mix}  
		\ll  
		\frac{k^2}{s\mu^4},  
		\]  
		up to constants and logarithmic factors.  
	\end{corollary}  
	
Thus the calibrated and uncalibrated observations fall into different algorithmic regimes: sparse mean recovery after synchronization versus sparse PCA without calibration.

	\subsection{One calibration, many downstream tasks}  
	
Calibration metadata concerns source gauges, not the downstream sparse direction. It can therefore be reused across tasks sharing the same source gauges and calibration graph.

	\begin{remark}[Task-independence and reuse]  
		\label{rem:task-independent-main}  
		Let \(V_1,\ldots,V_T\) be independent sparse directions, sampled independently of \((\sigma,E)\). If the metadata satisfies  
		\[  
		P(\mathfrak M\mid V_1,\ldots,V_T,\sigma,E)=P(\mathfrak M\mid \sigma,E),  
		\]  
		then $\MI(\mathfrak M;V_1,\ldots,V_T\mid \sigma,E)=0.$ Nevertheless, if the metadata induces a section \(q\) with \(M_q(g)=\Omega(n)\) and sufficient aligned source count, then the same \(q\) linearizes every task:  
		\[  
		X_q^{(t)}=\theta A_q^{(t)}V_t+Z^{(t)}.  
		\]  
	\end{remark}  
	
Thus calibration is reusable: it fixes a shared source convention without revealing any task-specific support.  

	The resulting algorithm, \GSL{} (\emph{Gauge--Synchronize--Linearize}), is    
immediate from the theory. It synchronizes the valid relative-gauge graph to    
obtain source weights \(\widehat u_\ell\). For sample \(i\), let    
\(e_i\in\{1,\ldots,L\}\) denote its source label, and set    
\(q_i=\widehat u_{e_i}\) on the downstream estimation samples.
No downstream samples are used to estimate these calibration weights. The algorithm then forms    
\[    
X_q=\frac{\sum_i q_iW_i}{\|q\|_2},    
\]    
and returns the top-\(k\) coordinates of \(|X_q|\). In the two-view model, where    
\(A_i,B_i\) are the two views and \(Y_i\in\{\pm1\}\) is the label,    
\(W_i=(A_i+Y_iB_i)/\sqrt2\), and the recovered direction is used in the bilinear    
classifier \(h(A,B)=\sgn((\widehat V^\top A)(\widehat V^\top B))\). Pseudocode is    
deferred to Appendix~\ref{app:deferred-main}.    

The exact separation assumes sign gauges, equal-magnitude sparse signals, Gaussian source amplitudes, isotropic Gaussian noise, and task-independent calibration. Appendix~\ref{app:noise-extension} gives a non-isotropic
Gaussian extension and clarifies the limits of the sparse-PCA hardness comparison.

	\section{Experiments}  
	\label{sec:experiments}  
	
Experiments check finite-sample predictions rather than prove the
experiment equivalence. Synthetic studies test the predicted recovery
scales; controlled real-feature studies test sign cancellation and
restoration. Separately acquired physical calibration and reuse are
reported in Appendix~\ref{sec:physical-calibration}.
Unless stated otherwise, \(\theta^2=2\mu^2\).
Synthetic recovery is measured by
\(\Corr=|\langle\widehat V,V\rangle|\) and support F1, with corresponding
success threshold \(0.9\).
The plotted strength is \(\rho_{\rm cal}=\theta^2n/\Ksp\).
Pseudocode is in Appendix~\ref{app:deferred-main};
experimental records and standard-error units are in
Appendix~\ref{app:records}.

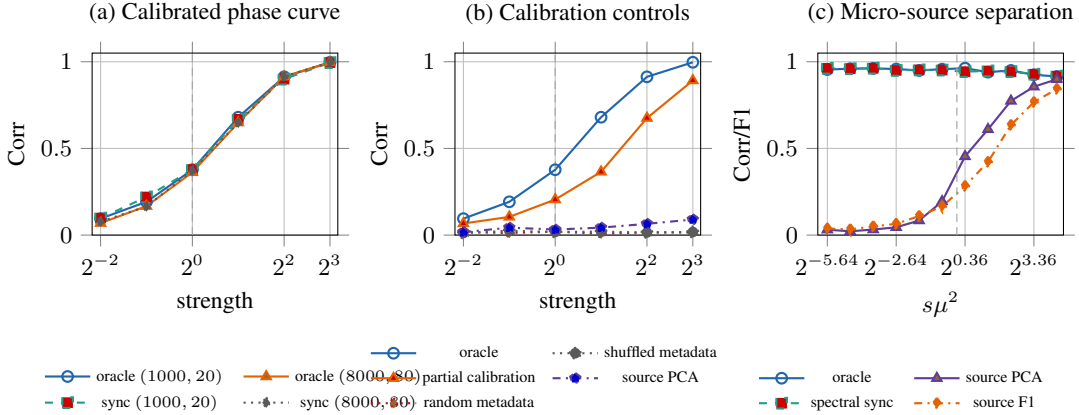
\begin{figure}[t]
	\centering
	\small % set base font size for all non-legend text in the figure
	
	\begin{tikzpicture}
		\begin{groupplot}[
			group style={
				group size=3 by 1,
				horizontal sep=1.55cm
			},
			width=3.25cm,
			height=2.41cm,                % <-- changed: 3.62 * 2/3 ≈ 2.41
			scale only axis,
			grid=both,
			tick align=outside,
			tick label style={font=\small},
			label style={font=\small},
			title style={font=\small, yshift=0.6mm},
			xmode=log,
			log basis x=2,
			ymin=0, ymax=1.05
			]
			
			% ===================== Panel 1 =====================
			\nextgroupplot[
			title={(a) Calibrated phase curve},
			xlabel={strength},
			ylabel={Corr},
			xmin=0.22, xmax=9,
			xtick={0.25,1,4,8},
			legend to name=legphase,
			legend columns=2,
			transpose legend,
			legend style={
				font=\tiny,
				draw=none,
				fill=none,
				/tikz/every even column/.append style={column sep=0.6em}
			}
			]
			\addplot+[color=calblue,mark=o,thick] coordinates {
				(0.25,0.095)(0.5,0.192)(1,0.377)(2,0.680)(4,0.913)(8,0.997)};
\addlegendentry{oracle \((1000,20)\)}
			
			\addplot+[color=syncgreen,mark=square*,thick,dashed] coordinates {
				(0.25,0.098)(0.5,0.220)(1,0.378)(2,0.668)(4,0.900)(8,0.995)};
\addlegendentry{sync \((1000,20)\)}
			
			\addplot+[color=partorange,mark=triangle*,thick] coordinates {
				(0.25,0.069)(0.5,0.167)(1,0.362)(2,0.649)(4,0.913)(8,0.996)};
\addlegendentry{oracle \((8000,80)\)}
			
			\addplot+[color=softgray,mark=diamond*,thick,dotted] coordinates {
				(0.25,0.078)(0.5,0.165)(1,0.372)(2,0.648)(4,0.915)(8,0.997)};
\addlegendentry{sync \((8000,80)\)}
			\draw[black!35,densely dashed,line width=0.45pt]  
			(axis cs:1,0) -- (axis cs:1,1.05);

			% ===================== Panel 2 =====================
			\nextgroupplot[
title={(b) Calibration controls},
			xlabel={strength},
ylabel={Corr},
			xmin=0.22, xmax=9,
			xtick={0.25,1,4,8},
			legend to name=legvalid,
			legend columns=3,
			transpose legend,
			legend style={
				font=\tiny,
				draw=none,
				fill=none,
				/tikz/every even column/.append style={column sep=0.45em}
			}
			]
			\addplot+[color=calblue,mark=o,thick] coordinates {
				(0.25,0.095)(0.5,0.192)(1,0.377)(2,0.680)(4,0.913)(8,0.997)};
			\addlegendentry{oracle}
			
			\addplot+[color=partorange,mark=triangle*,thick] coordinates {
				(0.25,0.067)(0.5,0.105)(1,0.205)(2,0.365)(4,0.675)(8,0.892)};
			\addlegendentry{partial calibration}
			
			\addplot+[color=badred,mark=diamond*,thick,dotted] coordinates {
				(0.25,0.017)(0.5,0.012)(1,0.020)(2,0.008)(4,0.015)(8,0.015)};
			\addlegendentry{random metadata}
			
			\addplot+[color=softgray,mark=*,thick,dotted] coordinates {
				(0.25,0.012)(0.5,0.025)(1,0.018)(2,0.020)(4,0.015)(8,0.018)};
			\addlegendentry{shuffled metadata}
			
			\addplot+[color=pcapurple,mark=pentagon*,thick,dashdotted] coordinates {
				(0.25,0.017)(0.5,0.043)(1,0.032)(2,0.043)(4,0.065)(8,0.090)};
			\addlegendentry{source PCA}
			
			\draw[black!35,densely dashed,line width=0.45pt]  
			(axis cs:1,0) -- (axis cs:1,1.05);  
			
			% ===================== Panel 3 =====================
			\nextgroupplot[
			title={(c) Micro-source separation},
			xlabel={\(s\mu^2\)},
ylabel={Corr/F1},
			xmin=0.015, xmax=25,
			xtick={0.02,0.16,1.28,10.24},
			legend to name=legmicro,
			legend columns=2,
			transpose legend,
			legend style={
				font=\tiny,
				draw=none,
				fill=none,
				/tikz/every even column/.append style={column sep=0.6em}
			}
			]
			\addplot+[color=calblue,mark=o,thick] coordinates {
				(0.02,0.955)(0.04,0.960)(0.08,0.962)(0.16,0.958)(0.32,0.950)(0.64,0.957)(1.28,0.963)(2.56,0.940)(5.12,0.950)(10.24,0.925)(20.48,0.915)};
			\addlegendentry{oracle}
			
			\addplot+[color=syncgreen,mark=square*,thick,dashed] coordinates {
				(0.02,0.962)(0.04,0.960)(0.08,0.963)(0.16,0.950)(0.32,0.953)(0.64,0.952)(1.28,0.943)(2.56,0.947)(5.12,0.943)(10.24,0.928)(20.48,0.917)};
			\addlegendentry{spectral sync}
			
			\addplot+[color=pcapurple,mark=triangle*,thick] coordinates {
				(0.02,0.031)(0.04,0.022)(0.08,0.032)(0.16,0.045)(0.32,0.084)(0.64,0.194)(1.28,0.454)(2.56,0.610)(5.12,0.774)(10.24,0.856)(20.48,0.899)};
			\addlegendentry{source PCA}
			
			\addplot+[color=partorange,mark=diamond*,thick,dashdotted] coordinates {
				(0.02,0.038)(0.04,0.033)(0.08,0.050)(0.16,0.065)(0.32,0.110)(0.64,0.165)(1.28,0.283)(2.56,0.422)(5.12,0.635)(10.24,0.767)(20.48,0.843)};
			\addlegendentry{source F1}
		
		\draw[black!35,densely dashed,line width=0.45pt]  
		(axis cs:1,0) -- (axis cs:1,1.05);

		\end{groupplot}
	\end{tikzpicture}
	
	\par\smallskip
	\noindent\hspace*{0.034\linewidth}%
	\begin{minipage}[t]{0.29\linewidth}
		\centering
		\pgfplotslegendfromname{legphase}
	\end{minipage}
	\hfill
	\begin{minipage}[t]{0.37\linewidth}
		\centering
		\pgfplotslegendfromname{legvalid}
	\end{minipage}
	\hfill
	\begin{minipage}[t]{0.29\linewidth}
		\centering
		\pgfplotslegendfromname{legmicro}
	\end{minipage}
	
\caption{{\small Mechanism tests. Calibration follows sparse-mean scaling;
		uninformative or mismatched metadata fails to restore alignment.
		The calibrated/source-label gap is largest for weak sources.
		Gray dashed lines mark unit strength.}}
	\label{fig:synthetic}
	\vspace{-0.4em}
\end{figure}

	\subsection{Synthetic mechanism tests}  

Figure~\ref{fig:synthetic}(a) tests the sparse-mean scaling predicted by
Theorem~\ref{thm:linearizes}: oracle calibration and spectral
synchronization nearly coincide across problem sizes.
Figure~\ref{fig:synthetic}(b) compares calibration quality:
informative records shift the transition, random and shuffled controls
remain near chance, and partial calibration gives an intermediate
transition. These controls need not violate task independence.
Figure~\ref{fig:synthetic}(c) fixes the calibrated first-order scale
and varies \(s\mu^2\). Calibrated recovery remains strong for weak
sources, whereas source-label PCA improves as the second-order
quantity \(\mu^4ns/\Ksp\) increases.

	\subsection{Additional mechanism checks}  
	
\paragraph{Aligned source count.}  
Theory separates signal from amplitude concentration. We fix \(M_q(g)=4096\) and vary \(\Nsrc\). Table~\ref{tab:aligned-source-main} confirms that \AlignedMass{} alone is not enough when concentrated on too few sources: as \(\Nsrc\) grows, \(A_q/\sqrt{M_q}\) stabilizes near the half-normal scale and recovery improves. This matches the role of \(\Nsrc\) in Theorem~\ref{thm:linearizes}: it controls whether source amplitudes concentrate around their mean rather than being dominated by a few lucky or unlucky sources.  

	\begin{table}[t]  
		\centering  
		\caption{At fixed \AlignedMass{}, increasing aligned source count stabilizes amplitude and improves recovery. Entries are mean \(\pm\) SE over \(100\) seeds.}  
		\label{tab:aligned-source-main}  
		\vspace{0.15em}  
		\begin{tabular}{cccc}  
			\toprule  
			\(\Nsrc\) & \(A_q/\sqrt{M_q}\) & Corr & Succ90 \\  
			\midrule  
			1   & \(0.825\pm0.068\) & \(0.632\pm0.038\) & \(0.460\pm0.050\) \\  
			4   & \(0.828\pm0.031\) & \(0.832\pm0.024\) & \(0.630\pm0.049\) \\  
			16  & \(0.783\pm0.015\) & \(0.880\pm0.011\) & \(0.600\pm0.049\) \\  
			32  & \(0.800\pm0.011\) & \(0.917\pm0.008\) & \(0.800\pm0.040\) \\  
			128 & \(0.790\pm0.005\) & \(0.916\pm0.006\) & \(0.770\pm0.042\) \\  
			\bottomrule  
		\end{tabular}  
		\vspace{-0.45em}  
	\end{table}  
	
	\paragraph{Reusable calibration.}  
	We test Remark~\ref{rem:task-independent-main} by synchronizing one relative-gauge graph before sampling any downstream task and reusing it for \(T=32\) independent sparse tasks. Table~\ref{tab:multitask-main-detailed} shows that the reused section matches oracle calibration; the downstream directions are distinct, so the gain is not due to task similarity. This illustrates the main distinction between calibration and supervision: the graph fixes a shared source convention once, and that convention can then be applied to new tasks without revealing their supports.  
	
\begin{table}[t]
	\centering
	\caption{One calibration graph reused across \(32\) independent sparse
		tasks. Entries are mean \(\pm\) task-level SE conditional on the graph.}
	\label{tab:multitask-main-detailed}
	\vspace{0.15em}
	\small
	\begin{tabular}{lcc}
		\toprule
		Method & Corr & F1 \\
		\midrule
		Oracle calibration
		& \(0.928{\pm}0.002\)
		& \(0.928{\pm}0.002\) \\
		Same sync reused
		& \(0.928{\pm}0.002\)
		& \(0.928{\pm}0.002\) \\
		No calibration
		& \(0.018{\pm}0.001\)
		& \(0.019{\pm}0.001\) \\
		\makecell{Random metadata\\reused}
		& \(0.020{\pm}0.002\)
		& \(0.021{\pm}0.002\) \\
		\makecell{Shuffled metadata\\reused}
		& \(0.022{\pm}0.001\)
		& \(0.023{\pm}0.001\) \\
		\makecell{Source diagonal\\PCA}
		& \(0.040{\pm}0.002\)
		& \(0.049{\pm}0.002\) \\
		\bottomrule
	\end{tabular}
	\vspace{-0.45em}
\end{table}
	
	\paragraph{Controlled real features.}    
Finally, we use real features and labels while synthetically imposing source gauges and calibration logs. Table~\ref{tab:real-feature-detailed-main} shows valid calibration restores clean no-gauge performance on digits and breast-cancer data, while random or shuffled metadata does not. These are controlled mechanism checks: gauges and calibration logs are imposed synthetically, while the feature distributions and labels are real. Source-label task alignment uses downstream labels and is therefore a supervised control.

	\begin{table}[t]  
		\centering  
		\caption{Controlled real-feature experiments. Features and labels are real; source gauges and calibration metadata are imposed. Entries are mean \(\pm\) SE over \(30\) seeds.}  
		\label{tab:real-feature-detailed-main}  
		\vspace{0.15em}  
		\small  
		\setlength{\tabcolsep}{2.5pt}  
		\begin{tabular}{lcccc}  
			\toprule  
			Method & Digits Acc. & Digits \(M_q/n\) & Breast Acc. & Breast \(M_q/n\) \\  
			\midrule  
			Clean no gauge  
			& \(0.968\pm0.001\)  
			& \(1.000\pm0.000\)  
			& \(0.933\pm0.003\)  
			& \(1.000\pm0.000\) \\  
			Oracle calibration  
			& \(0.968\pm0.001\)  
			& \(1.000\pm0.000\)  
			& \(0.933\pm0.003\)  
			& \(1.000\pm0.000\) \\  
			Relative sync calibration  
			& \(0.968\pm0.001\)  
			& \(1.000\pm0.000\)  
			& \(0.932\pm0.003\)  
			& \(0.997\pm0.003\) \\  
			Noisy abs. calibration, \(10\%\) flips  
			& \(0.874\pm0.006\)  
			& \(0.650\pm0.018\)  
			& \(0.851\pm0.011\)  
			& \(0.667\pm0.031\) \\  
			No calibration  
			& \(0.500\pm0.004\)  
			& \(0.003\pm0.000\)  
			& \(0.501\pm0.005\)  
			& \((1.180\pm0.234)\times10^{-5}\) \\  
			Random metadata  
			& \(0.525\pm0.005\)  
			& \(0.011\pm0.002\)  
			& \(0.556\pm0.008\)  
			& \(0.022\pm0.005\) \\  
			Shuffled metadata  
			& \(0.533\pm0.007\)  
			& \(0.016\pm0.003\)  
			& \(0.557\pm0.007\)  
			& \(0.024\pm0.005\) \\  
			Source-label task alignment  
			& \(0.818\pm0.008\)  
			& \(0.704\pm0.018\)  
			& \(0.933\pm0.003\)  
			& \(1.000\pm0.000\) \\  
			\bottomrule  
		\end{tabular}  
		\vspace{-0.45em}  
	\end{table}

\section{Related work}  

Our work connects sparse PCA, synchronization, orbit recovery, and metadata-aware robust learning. Sparse PCA is a standard model for high-dimensional statistical and computational gaps \citep{johnstone2009sparse,amini2009high,berthet2013complexity,deshpande2014sparse}; here it arises instead as the exact experiment induced by source-labelled uncalibrated observations. Our calibration step is closely related to synchronization over groups such as \(\mathbb Z_2\), \(SO(d)\), and \(U(1)\) \citep{singer2011angular,bandeira2013cheeger}, with additional connections to community-detection and graph-based recovery phenomena \citep{abbe2017community}; here synchronization is used to turn task-independent gauge metadata into restored first-order signal. The model is also related to orbit recovery and multi-reference alignment, where latent group actions suppress first-order information \citep{bendory2017bispectrum}, and to related high-order-moment phenomena in spiked models \citep{perry2019sample}. More broadly, the paper complements work on invariant learning, robustness under group or domain shift, and metadata-aware pipelines \citep{arjovsky2019invariant,sagawa2020distributionally,koh2021wilds,zhou2024dgsurvey,zhou2026spurious,zhou2022rethinking}. Related phenomena arise in attribute recognition, where label co-occurrence bias and heterogeneous feature distributions induce spurious correlations that downstream methods must remove or balance \citep{ijcai2023p203,pmlr-v235-zhou24j,zhou2025cooccurrence,zhou2025heterogeneous}. Our focus is narrower: metadata that identifies nuisance coordinate conventions rather than providing downstream supervision.    

\section{Conclusion}  

This paper isolates a basic distinction in multi-source learning. Source labels tell us which samples belong together, but they do not tell us how different sources orient the same latent signal. In our stylized sparse model, that distinction changes the problem completely. Without calibration, the experiment is sparse PCA; with valid relative-gauge metadata, synchronization restores a first-order sparse mean signal. The aligned-mass viewpoint makes this change measurable and leads to clear statistical and computational separations in weak-source regimes. More broadly, the paper highlights a simple principle: metadata can matter not because it labels the task, but because it fixes the coordinate system in which the task becomes visible. This suggests treating calibration as a distinct resource from supervision or grouping information. Extending this viewpoint to richer transformation groups and approximately valid metadata is a natural next step.  

\begin{ack}
	We declare that this work received no external funding and that the authors have no competing interests to disclose.
\end{ack}
	
	\bibliographystyle{plainnat}  
	\bibliography{ref} 
	
	\clearpage
	\appendix  
	
	\makeatletter  
	\renewcommand{\theHfigure}{A.\arabic{figure}}  
	\renewcommand{\theHtable}{A.\arabic{table}}  
	\renewcommand{\theHalgorithm}{A.\arabic{algorithm}}  
	\makeatother  
	
\section{Deferred algorithm}  

	\label{app:deferred-main}  
	
	\subsection{\GSL{} pseudocode}  
	
	\begin{algorithm}[H]    
	\caption{\GSL{}: Gauge--Synchronize--Linearize}    
	\label{alg:gsl}    
	\begin{algorithmic}[1]    
		\Require Linearization samples \((W_i,e_i)\), where \(e_i\in\{1,\ldots,L\}\) is the source label of sample \(i\); valid calibration metadata \(\mathfrak M\); sparsity \(k\)    
		\State Construct a signed source calibration graph from \(\mathfrak M\)    
		\State Run \(\mathbb Z_2\) synchronization to obtain source weights \(\widehat u_\ell\)    
\State Set \(q_i=\widehat u_{e_i}\) for each downstream sample \(i\)
		\State Compute \(X_q=\sum_i q_iW_i/\|q\|_2\)    
		\State Top-\(k\) threshold \(|X_q|\) to obtain \(\widehat V\)    
		\State \Return \(\widehat V\)    
	\end{algorithmic}    
\end{algorithm}    

		In the two-view setting of Appendix~\ref{app:two-view}, the linearization samples are obtained from    
	\(W_i=(A_i+Y_iB_i)/\sqrt2\), where \(A_i\) and \(B_i\) are the two views and    
	\(Y_i\in\{\pm1\}\) is the label. After recovering \(\widehat V\), prediction uses    
	\(h(A,B)=\sgn((\widehat V^\top A)(\widehat V^\top B))\).

	\section{Two-view origin}  
	\label{app:two-view}

The source-Gaussian statistic arises from a simple two-view classification model. Let \(a_\ell\stackrel{\mathrm{i.i.d.}}{\sim}N(0,1)\), and let  
\[    
\xi_{\ell r},\zeta_{\ell r}\stackrel{\mathrm{i.i.d.}}{\sim} N(0,I_d),    
\qquad    
Y_{\ell r}\sim \Unif\{\pm1\},    
\]    
with all variables independent of \(a_\ell\) and \(V\).  
	 Define  
	\[  
	A_{\ell r}=\mu a_\ell V+\xi_{\ell r},  
	\qquad  
	B_{\ell r}=\mu Y_{\ell r}a_\ell V+\zeta_{\ell r}.  
	\]  
	Then  
	\[  
	W_{\ell r}  
	=  
	\frac{A_{\ell r}+Y_{\ell r}B_{\ell r}}{\sqrt2}  
	=  
	\theta a_\ell V+\varepsilon_{\ell r},  
	\qquad  
	\theta=\sqrt2\,\mu,  
	\]  
	where  
	\[  
	\varepsilon_{\ell r}=\frac{\xi_{\ell r}+Y_{\ell r}\zeta_{\ell r}}{\sqrt2}\sim N(0,I_d).  
	\]  
	The orthogonal combination  
	\[  
	\frac{A_{\ell r}-Y_{\ell r}B_{\ell r}}{\sqrt2}  
	\]  
	is pure Gaussian noise. Recovering \(V\) in the main text therefore corresponds to recovering the discriminative sparse direction used by the bilinear classifier.  
	
	\section{Proofs}  
	\label{app:proofs}  
	
	\subsection{Proof of Proposition~\ref{prop:alignment-capacity}}  
	
	Fix \((\mathfrak M,E)\).  For a sample-level weighting \(q\), let  
	\[  
	Q_\ell=\sum_{r=1}^s q_{\ell r}.  
	\]  
	Then  
	\[  
	q^\top g=\sum_{\ell=1}^L Q_\ell\sigma_\ell=Q^\top\sigma.  
	\]  
	By Cauchy's inequality within each source,  
	\[  
	Q_\ell^2  
	\le  
	s\sum_{r=1}^s q_{\ell r}^2,  
	\]  
	and hence  
	\[  
	\norm{q}_2^2  
	\ge  
	\frac{1}{s}\norm{Q}_2^2.  
	\]  
	If \(Q=0\), the expected aligned mass is zero, so the desired upper bound
	is immediate. We therefore assume \(Q\neq0\) in the following Rayleigh
	quotient calculation. Therefore, conditional on \((\mathfrak M,E)\),  
	\[  
	\E[M_q(g)\mid \mathfrak M,E]  
	=  
	\frac{Q^\top \Gamma(\mathfrak M,E)Q}{\norm{q}_2^2}  
	\le  
	s\,  
	\frac{Q^\top \Gamma(\mathfrak M,E)Q}{\norm{Q}_2^2}  
	\le  
	s\,\lambda_{\max}(\Gamma(\mathfrak M,E)).  
	\]  
	The upper bound is attained by taking \(Q\) to be a top eigenvector of  
	\(\Gamma(\mathfrak M,E)\) and spreading the weight uniformly within each  
	source, for instance \(q_{\ell r}=Q_\ell/s\).  This gives  
	\[  
	\E[M_q(g)\mid \mathfrak M,E]  
	=  
	s\,\lambda_{\max}(\Gamma(\mathfrak M,E)).  
	\]  
	This proves the claim.  
	
	\subsection{Bayes-to-minimax consequence for fixed gauges}  
	
	Let \(S(\widehat V,V)\) denote the success event, for example  
	\[  
	S(\widehat V,V)  
	=  
	\{\mathrm{SupportF1}(\widehat V,V)\ge 1-\tau\}.  
	\]  
	Suppose that for an arbitrary estimator \(\widehat V\),  
	\[  
	\E_{\sigma}  
	\Pr\!\left(S(\widehat V,V)\mid \sigma\right)  
	\le  
	\frac12,  
	\]  
	where the outer expectation is over the uniform Rademacher prior on  
	\(\sigma\).  Then there exists at least one fixed gauge vector \(\sigma^\star\)  
	such that  
	\[  
	\Pr\!\left(S(\widehat V,V)\mid \sigma=\sigma^\star\right)  
	\le  
	\frac12.  
	\]  
	Equivalently,  
	\[  
	\Pr\!\left(S(\widehat V,V)^c\mid \sigma=\sigma^\star\right)  
	\ge  
	\frac12.  
	\]  
	Since this holds for every estimator, the Bayes lower bound under the  
	Rademacher gauge prior implies the corresponding minimax lower bound over  
	fixed unknown gauges.  
	
	\subsection{Unequal source sizes}  
	
	For completeness we record the source-label aligned-mass bound for unequal  
	source sizes.  Let source \(\ell\) contain \(s_\ell\) samples and let  
	\(s_{\max}=\max_\ell s_\ell\).  If \(q\) is measurable with respect to source  
	labels and independent of the unknown gauges, then conditional on \(q\),  
	\[  
	q^\top g  
	=  
	\sum_{\ell=1}^L Q_\ell\sigma_\ell,  
	\qquad  
	Q_\ell=\sum_{r=1}^{s_\ell}q_{\ell r}.  
	\]  
	Hoeffding's inequality gives  
	\[  
	\Pbb_\sigma\left(  
	\left|\sum_{\ell=1}^L Q_\ell\sigma_\ell\right|\ge t  
	\ \middle|\ q  
	\right)  
	\le  
	2\exp\left(  
	-\frac{t^2}{2\sum_\ell Q_\ell^2}  
	\right).  
	\]  
	Moreover,  
	\[  
	Q_\ell^2  
	\le  
	s_\ell\sum_{r=1}^{s_\ell}q_{\ell r}^2  
	\le  
	s_{\max}\sum_{r=1}^{s_\ell}q_{\ell r}^2,  
	\]  
	so  
	\[  
	\sum_\ell Q_\ell^2  
	\le  
	s_{\max}\norm q_2^2.  
	\]  
	Taking  
	\[  
	t=\norm q_2\sqrt{2s_{\max}\log(2/\delta)}  
	\]  
	yields  
	\[  
	\Pbb_\sigma\left(  
	M_q(g)  
	\le  
	2s_{\max}\log\frac{2}{\delta}  
	\ \middle|\ q  
	\right)  
	\ge  
	1-\delta.  
	\]  
	Thus source-label-measurable first-order sections cannot produce extensive  
	\AlignedMass{} unless one source already contains a constant fraction of the  
	samples.

	\subsection{Proof of Proposition~\ref{prop:source-spca}}  
	
	Fix a source \(\ell\). Since \(W_{\ell r}=\theta a_\ell V+\varepsilon_{\ell r}\),  
	\[  
	X_\ell=\frac1{\sqrt s}\sum_{r=1}^s W_{\ell r}  
	=  
	\theta\sqrt{s}\,a_\ell V+Z_\ell,  
	\qquad  
	Z_\ell\sim N(0,I_d).  
	\]  
	Because \(a_\ell\sim N(0,1)\),  
	\[  
	X_\ell\sim N(0,I_d+\theta^2sVV^\top).  
	\]  
	Now decompose the \(s\) samples within each source into the all-ones direction and its orthogonal complement. The all-ones direction is exactly \(X_\ell\). Every orthogonal contrast cancels the common source signal and is pure Gaussian noise independent of \(V\). This proves equivalence to sparse PCA samples plus \(V\)-independent Gaussian residuals.  
	
	\subsection{A standard sparse-mean recovery lemma}  
	
	\begin{lemma}[Top-\(k\) sparse mean recovery]  
		\label{lem:sparse-mean}  
		Let  
		\[  
		X=\gamma V+Z,  
		\qquad  
		Z\sim N(0,I_d),  
		\qquad  
		V\in\cV_{d,k}.  
		\]  
		Let \(\widehat V\in\arg\max_{v\in\cV_{d,k}}|\ip{X}{v}|\), equivalently top-\(k\) threshold \(|X_j|\) and use observed signs. There is a universal constant \(C>0\) such that, for \(0<\varepsilon,\delta<1\), if  
		\[  
		\gamma^2  
		\ge  
		C\varepsilon^{-2}  
		\left(k\log\frac{2ed}{k}+\log\frac{2}{\delta}\right),  
		\]  
		then  
		\[  
		\Pbb\left(|\ip{\widehat V}{V}|\ge 1-\varepsilon\right)\ge 1-\delta.  
		\]  
	\end{lemma}  
	
	\begin{proof}  
		Let  
		\[  
		\eta=\sup_{v\in\cV_{d,k}}|\ip{Z}{v}|.  
		\]  
		Since  
		\[  
|\cV_{d,k}|    
\le    
\binom{d}{k}2^k    
\le    
\left(\frac{2ed}{k}\right)^k,  
		\]  
		a union bound over Gaussian tails gives, with probability at least \(1-\delta\),  
		\[  
		\eta  
		\le  
		C\sqrt{k\log\frac{2ed}{k}+\log\frac{2}{\delta}}.  
		\]  
		On this event, optimality of \(\widehat V\) implies  
		\[  
		|\ip{X}{\widehat V}|  
		\ge  
		|\ip{X}{V}|  
		\ge  
		|\gamma|-\eta,  
		\]  
		while also  
		\[  
		|\ip{X}{\widehat V}|  
		\le  
		|\gamma|\,|\ip{V}{\widehat V}|+\eta.  
		\]  
		Hence  
		\[  
		|\ip{V}{\widehat V}|  
		\ge  
		1-\frac{2\eta}{|\gamma|}.  
		\]  
		The stated condition ensures \(2\eta/|\gamma|\le \varepsilon\).  
	\end{proof}  
	
	\subsection{Weighted half-normal concentration}  
	
	\begin{lemma}[Weighted half-normal concentration]  
		\label{lem:half-normal}  
Let \(G_1,\ldots,G_L\stackrel{\mathrm{i.i.d.}}{\sim}N(0,1)\),
set \(R_\ell=|G_\ell|\), and let \(w_\ell\) be fixed real weights. Put  
		\[  
		S=\sum_\ell w_\ell,  
		\qquad  
		B^2=\sum_\ell w_\ell^2,  
		\qquad  
		\kappa_1=\E R_\ell=\sqrt{2/\pi}.  
		\]  
		There exist universal constants \(c,C>0\) such that  
		\[  
		\Pbb\left(  
		\left|  
		\sum_\ell w_\ell(R_\ell-\kappa_1)  
		\right|\ge t  
		\right)  
		\le  
		2\exp\left(-c\frac{t^2}{B^2}\right)  
		\]  
		for all \(t\ge 0\). Consequently, if \(S^2/B^2\ge C\log(2/\delta)\), then with probability at least \(1-\delta\),  
		\[  
		\left|\sum_\ell w_\ell R_\ell\right|  
		\ge  
		\frac{\kappa_1}{2}|S|.  
		\]  
	\end{lemma}  
	
\begin{proof}
	The case \(B=0\) is trivial. Otherwise, the function
	\[
	F(G_1,\ldots,G_L)=\sum_\ell w_\ell|G_\ell|
	\]
	is \(B\)-Lipschitz in Euclidean norm, and
	\(\E F=\kappa_1S\).
	Gaussian concentration for Lipschitz functions
	\citep{boucheron2013concentrationbook} therefore gives
	\[
	\Pbb\bigl(|F-\kappa_1S|\ge t\bigr)
	\le
	2\exp\!\left(-\frac{t^2}{2B^2}\right).
	\]
	Taking \(t=\kappa_1|S|/2\) and using
	\(|F|\ge\kappa_1|S|-|F-\kappa_1S|\) proves the second claim whenever
	\(S^2/B^2\ge C\log(2/\delta)\).
\end{proof}
	
	\subsection{Proof of Theorem~\ref{thm:linearizes}}  
	
	Conditional on \(q,\sigma,R,V\),  
	\[  
	X_q  
	=  
	\theta  
	\frac{\sum_{\ell,r} q_{\ell r}R_\ell\sigma_\ell}{\norm{q}_2}V  
	+  
	\frac{\sum_{\ell,r} q_{\ell r}\varepsilon_{\ell r}}{\norm{q}_2}.  
	\]  
Because valid calibration makes \(q\) independent of the downstream feature
noise conditional on the source signs and design,
	\[  
	\frac{\sum_{\ell,r} q_{\ell r}\varepsilon_{\ell r}}{\norm{q}_2}\sim N(0,I_d).  
	\]  
	Thus \(X_q\) has the sparse-mean form.  
	
	In the constant-amplitude model,  
	\[  
	A_q^2  
	=  
	\frac{(q^\top g)^2}{\norm{q}_2^2}  
	=  
	M_q(g).  
	\]  
	For source-Gaussian amplitudes, write \(Q_\ell=\sum_r q_{\ell r}\) and \(w_\ell=Q_\ell\sigma_\ell\). Then  
	\[  
	M_q(g)=\frac{S^2}{\norm{q}_2^2},  
	\qquad  
	\Nsrc(q,g)=\frac{S^2}{B^2},  
	\]  
	where \(S=\sum_\ell w_\ell\) and \(B^2=\sum_\ell w_\ell^2\). By Lemma~\ref{lem:half-normal},  
	\[  
	\left|\sum_\ell w_\ell R_\ell\right|  
	\ge  
	\frac{\kappa_1}{2}|S|  
	\]  
	with probability at least \(1-\delta\) whenever \(\Nsrc(q,g)\gtrsim \log(2/\delta)\). Therefore    
	\[    
	|A_q|    
	\ge    
	\frac{\kappa_1}{2}\sqrt{M_q(g)}.    
	\]    
The factor \(1/4\) in \(\theta^2A_q^2\) is absorbed into \(\gtrsim\). Using failure probability \(\delta/2\) for the half-normal event and \(\delta/2\) for Lemma~\ref{lem:sparse-mean}, and adjusting constants, completes the proof.

	\subsection{Proof of Lemma~\ref{lem:no-section-main}}  
	
	Conditional on \(q\), define \(Q_\ell=\sum_{r=1}^s q_{\ell r}\). Then  
	\[  
	q^\top g=\sum_{\ell=1}^L Q_\ell\sigma_\ell.  
	\]  
	This is a Rademacher sum. Hoeffding's inequality gives  
	\[  
	\Pbb_\sigma\left(  
	\left|\sum_{\ell=1}^L Q_\ell\sigma_\ell\right|\ge t  
	\ \middle|\ q  
	\right)  
	\le  
	2\exp\left(-\frac{t^2}{2\sum_\ell Q_\ell^2}\right).  
	\]  
	By Cauchy,  
	\[  
	\sum_\ell Q_\ell^2  
	\le  
	s\sum_{\ell,r} q_{\ell r}^2  
	=  
	s\norm{q}_2^2.  
	\]  
	Taking \(t=\norm{q}_2\sqrt{2s\log(2/\delta)}\) yields the claim.  
	
	\subsection{Proof of Lemma~\ref{lem:spectral-sync-main}}
	
	Conditional on \(\sigma\), put
	\[
	H=B-\E[B\mid\sigma].
	\]
	The upper-triangular entries of \(H\) are independent and centered.
	Moreover,
	\[
	|H_{\ell m}|\le2,
	\qquad
	\max_\ell\sum_m\E[H_{\ell m}^2\mid\sigma]\le pL.
	\]
	We use the sharp bounded-entry spectral-norm estimate for symmetric
	matrices with independent centered upper-triangular entries
	\citep{bandeira2016sharpnorm}. In the present notation, its bounded-entry
	consequence, with symmetrization if the entry distributions are not
	symmetric, gives
	\[
	\Pbb\!\left(
	\norm H_{\mathrm{op}}
	>
	C\left(\sqrt{pL}+\sqrt{\log(2L/\delta)}\right)
	\,\middle|\,\sigma
	\right)
	\le\delta.
	\]
	Here \(\norm{\cdot}_{\mathrm{op}}\) is the operator norm.
	Under \(pL\ge C\log(L/\delta)\), this yields
	\[
	\norm H_{\mathrm{op}}\le C\sqrt{pL}
	\]
	with probability at least \(1-\delta\).
	This sharp bound, rather than a direct application of standard matrix
	Bernstein, is what removes the multiplicative logarithmic factor.
	
	The population matrix
	\[
	\E[B\mid\sigma]=p\alpha(\sigma\sigma^\top-I)
	\]
	has leading eigenvector \(\sigma/\sqrt L\), leading eigenvalue
	\(p\alpha(L-1)\), and all remaining eigenvalues equal to \(-p\alpha\).
	Its eigengap is therefore exactly \(\alpha pL\).
	If \(\alpha^2pL\) exceeds a sufficiently large universal constant, then
	\(\norm H_{\mathrm{op}}\le\alpha pL/4\).
	The Davis--Kahan eigenvector perturbation bound
	\citep{yu2015usefuldk} consequently gives
	\[
	1-\beta^2
	\le
	C\frac{\norm H_{\mathrm{op}}^2}{(\alpha pL)^2}
	\le
	\frac{C}{\alpha^2pL}.
	\]
	This proves the lemma without changing its stated degree conditions.
	
	\subsection{Proof of Theorem~\ref{thm:source-fano}}
	
	\paragraph{Support packing.}
	Write \([d]=\{1,\ldots,d\}\), and set \(\rho=1/64\).
	A greedy construction on all \(k\)-subsets of \([d]\), deleting after
	each selection every support obtained by replacing at most
	\(\lfloor\rho k\rfloor\) coordinates, gives supports
	\(S_1,\ldots,S_M\) satisfying
	\[
	|S_i|=k,
	\qquad
	|S_i\triangle S_j|>2\rho k
	\quad(i\neq j),
	\]
	and
	\[
	M
	\ge
	\frac{\binom dk}
	{\displaystyle
		\sum_{j=0}^{\lfloor\rho k\rfloor}
		\binom kj\binom{d-k}{j}}.
	\]
	Here \(\triangle\) denotes symmetric difference.
	For completeness, let
	\(h(x)=-x\log x-(1-x)\log(1-x)\).
	The denominator is at most
	\[
	\exp\!\left(
	kh(\rho)+(d-k)h\!\left(\frac{\rho k}{d-k}\right)
	\right)
	\le
	\exp\!\left(
	\rho k\log(d/k)+2\rho k\log(e/\rho)
	\right).
	\]
	Using \(\binom dk\ge(d/k)^k\) and \(d/k\ge2\), we obtain
	\[
	\log M
	\ge
	(1-\rho)k\log(d/k)-2\rho k\log(e/\rho)
	\ge
	c_1\Ksp
	\]
	for a universal \(c_1>0\), for example \(c_1=1/4\).
	Choose \(C_0\) sufficiently large that \(\log M\ge4\log2\).
	Let \(J\) be uniform on \(\{1,\ldots,M\}\) and set
	\[
	V_J=k^{-1/2}\ind_{S_J}\in\cV_{d,k}.
	\]
	
	\paragraph{From high F1 to index recovery.}
	For any estimated support \(\widehat S\), define
	\[
	\mathrm{F1}(\widehat S,S_J)
	=
	\frac{2|\widehat S\cap S_J|}{|\widehat S|+k}.
	\]
	If this is at least \(1-\tau\), then, writing
	\(D=|\widehat S\triangle S_J|\),
	\[
	D\le\tau(|\widehat S|+k)
	\le\tau(2k+D),
	\]
	so
	\[
	D\le\frac{2\tau k}{1-\tau}.
	\]
	Choose the universal constant \(\tau>0\) so that
	\(2\tau/(1-\tau)<\rho\).
	Then \(D<\rho k\), whereas distinct packing supports have distance
	greater than \(2\rho k\).
	Thus nearest-support decoding in symmetric-difference distance uniquely
	recovers \(J\). This argument does not require
	\(|\widehat S|=k\).
	
	\paragraph{Information bound and Fano.}
	By Proposition~\ref{prop:source-spca}, the observations reduce to
	\(L\) independent samples with covariance
	\(I_d+\lambda V_JV_J^\top\), together with independent residual noise
	whose distribution does not depend on \(J\).
	Let \(P_j\) denote the aggregate-sample distribution under \(J=j\), and
	let \(P_0=N(0,I_d)^{\otimes L}\).
	The Gaussian KL formula gives
	\[
	\KL(P_j\|P_0)
	=
	\frac L2
	\left[
	\operatorname{tr}(\lambda V_jV_j^\top)
	-\log\det(I_d+\lambda V_jV_j^\top)
	\right]
	=
	\frac L2[\lambda-\log(1+\lambda)].
	\]
	Consequently,
	\[
	\MI(J;\text{observations})
	=
	\MI(J;X_{1:L})
	\le
	\frac1M\sum_{j=1}^M\KL(P_j\|P_0)
	=
	I_{\rm src}.
	\]
	Fano's inequality \citep{cover2006elementsfano} implies, for every
	index decoder \(\widehat J\),
	\[
	\Pbb(\widehat J=J)
	\le
	\frac{I_{\rm src}+\log2}{\log M}.
	\]
	If \(I_{\rm src}\le c\Ksp\) with \(c\le c_1/4\), the right-hand side is
	at most \(1/2\), since \(\log M\ge4\log2\).
	The decoding implication above proves the high-F1 lower bound.
	
	Finally, for \(0\le\lambda\le1\),
	\[
	\frac{\lambda^2}{4}
	\le
	\lambda-\log(1+\lambda)
	=
	\int_0^\lambda\frac{t}{1+t}\,dt
	\le
	\frac{\lambda^2}{2}.
	\]
	Hence
	\[
	I_{\rm src}\asymp L\lambda^2
	=
	\theta^4ns
	\asymp\mu^4ns,
	\]
	which gives the stated weak-source condition.
	
	\subsection{Proof of Remark~\ref{rem:task-independent-main}}  
	
	The mutual-information statement follows from conditional independence:  
	\[  
	P(\mathfrak M,V_{1:T}\mid \sigma,E)  
	=  
	P(\mathfrak M\mid \sigma,E)P(V_{1:T}\mid \sigma,E).  
	\]  
	Therefore \(\MI(\mathfrak M;V_{1:T}\mid \sigma,E)=0\). For each task \(t\), conditional on the same metadata-induced section \(q\), Theorem~\ref{thm:linearizes} applies with \(V=V_t\), giving  
	\[  
	X_q^{(t)}=\theta A_q^{(t)}V_t+Z^{(t)}.  
	\]  
	
	\subsection{Proof of Corollary~\ref{cor:conditional-separation-main}}  
	
	The calibrated claim follows from Theorem~\ref{thm:linearizes} with \(M_q(g)=\beta^2n\). Source-label statistical insufficiency follows from Theorem~\ref{thm:source-fano}. For larger uncalibrated sample size \(n_{\rm mix}\), the compressed source experiment has  
	\[  
	N=\frac{n_{\rm mix}}{s},  
	\qquad  
	\lambda=\theta^2s,  
	\qquad  
	N\lambda^2=\theta^4n_{\rm mix}s.  
	\]  
	The stated window is exactly the sparse-PCA hardness window under Conjecture~\ref{conj:spca-main}, up to constants and logarithmic factors. High-overlap recovery would imply detection by splitting independent
	sources into two groups, estimating \(\widehat V\) from the source
	aggregates in one group, and testing the variance along \(\widehat V\)
	using the independent aggregates in the other group.
	The split is across sources, not across repeated observations sharing
	the same source amplitude.
	
	\section{Noise extension and scope}
	\label{app:noise-extension}
	
	The first-order versus second-order distinction has a limited extension
	beyond isotropic noise. Suppose independently across samples
	\[
	\varepsilon_{\ell r}\sim N(0,\Sigma_\ell),
	\]
	where the source covariances \(\Sigma_\ell\) are fixed and do not depend
	on \(V\). Retain the source-amplitude model and valid calibration
	assumptions from the main text. Conditional on the calibration weights
	and source variables,
	\[
	X_q=\theta A_qV+Z_q,
	\qquad
	Z_q\sim N(0,\Sigma_q),
	\qquad
	\Sigma_q
	=
	\frac{\sum_{\ell,r}q_{\ell r}^2\Sigma_\ell}{\norm q_2^2}.
	\]
	Thus calibration still restores a first-order mean signal, although its
	noise covariance is no longer the identity.
	
	\paragraph{A sufficient support-recovery bound.}
	Let
	\[
	v_q=\max_j(\Sigma_q)_{jj}.
	\]
	For \(0<\delta<1/2\), a sufficient condition for exact support recovery
	by top-\(k\) thresholding is
	\[
	\Nsrc(q,g)\gtrsim\log(4/\delta),
	\qquad
	\theta^2M_q(g)\gtrsim k v_q\log(4d/\delta).
	\]
	To see this, coordinatewise Gaussian tails and a union bound give
	\[
	\Pbb\!\left(
	\max_j |(Z_q)_j|
	>
	\sqrt{2v_q\log(4d/\delta)}
	\,\middle|\,q,\sigma,R,V,E
	\right)
	\le\delta/2.
	\]
	No independence between feature coordinates is needed for this step.
	Lemma~\ref{lem:half-normal}, with failure probability \(\delta/2\),
	also gives
	\[
	|A_q|\ge\frac{\kappa_1}{2}\sqrt{M_q(g)}
	\]
	under the stated aligned-source-count condition.
	Since every nonzero entry of \(V\) has magnitude \(1/\sqrt{k}\),
	a sufficiently large universal constant in the signal condition makes
	every signal coordinate more than twice the displayed noise bound.
	The top-\(k\) support is therefore correct with probability at least
	\(1-\delta\).
	
	If
	\[
	\norm{\Sigma_\ell-I_d}_{\mathrm{op}}\le\eta
	\qquad\text{for every source }\ell,
	\]
	then \(\Sigma_q\) is a convex combination of the source covariances,
	so
	\[
	\norm{\Sigma_q-I_d}_{\mathrm{op}}\le\eta,
	\qquad
	v_q\le1+\eta.
	\]
	Here \(\norm{\cdot}_{\mathrm{op}}\) denotes the operator norm.
	
	\paragraph{The uncalibrated experiment.}
	Without calibration, the source aggregate has distribution
	\[
	X_\ell
	\sim
	N(0,\Sigma_\ell+\theta^2sVV^\top).
	\]
	The remaining \(s-1\) orthogonal within-source contrasts are independent
	background observations with distribution \(N(0,\Sigma_\ell)\), and
	are independent of \(X_\ell\). Their distributions do not depend on
	\(V\), but they can help estimate unknown source covariances.
	The uncalibrated experiment is therefore a generalized,
	possibly heteroscedastic spiked-covariance experiment with
	background-only observations, not standard identity-background sparse
	PCA. Its dependence on \(V\) remains second-order, but sharp thresholds
	and computational hardness for arbitrary unknown \(\Sigma_\ell\)
	are not established here.
	
	Calibration validity concerns the information channel producing \(q\);
	noise covariance concerns the downstream observations after \(q\) is
	fixed. These are separate assumptions. This extension does not establish
	robustness to task-dependent metadata, approximate sparsity, or richer
	source transformations.
	
	\section{Additional experimental details}  
	\label{app:records}  
	
	Simulation records use deterministic base seed
	\[
	\texttt{base\_seed}=12345
	\]
	and per-cell seeds generated by
	\[
	\texttt{stable\_seed(base\_seed, experiment, configuration, seed, method)}.
	\]
	Unless stated otherwise, repeated-run cells use \(B=30\) independent
	randomized runs. The aligned-source-count experiment uses \(B=100\).
	The controlled real-feature experiments use seeds
	\(12345,\ldots,12374\).
	
	\paragraph{Role of the experiments.}
	The methods are calibration-information controls rather than claims
	of general algorithmic superiority: oracle signs represent perfect
	calibration, relative synchronization represents noisy calibration,
	source labels alone represent the uncalibrated experiment, and random
	or shuffled records represent uninformative or mismatched metadata.
	Synthetic experiments test finite-sample scaling predictions.
	The scikit-learn studies test cancellation and restoration on real
	features with imposed source signs. The physical study tests
	separately acquired calibration and task reuse.
	These experiments do not prove the statistical equivalence in
	Proposition~\ref{prop:source-spca}.
	
	\paragraph{Meaning and unit of SE.}
	SE means standard error, not squared error.
	If \(m_b\) is the aggregated metric from randomized run \(b\), then
	\[
	\mathrm{SE}
	=
	\frac{\operatorname{sd}(m_1,\ldots,m_B)}{\sqrt B},
	\]
	where \(\operatorname{sd}\) is the sample standard deviation.
	For digits, each \(m_b\) first averages test accuracy across the five
	digit-pair tasks; SE is then computed across the \(30\) run-level
	averages. For breast cancer, it is computed across \(30\)
	randomized-run test accuracies.
	
	The reuse experiment instead fixes one calibration graph and reports
	mean and SE across \(32\) independently sampled downstream tasks,
	conditional on that shared graph. These are not \(30\)-seed estimates
	or uncertainty across calibration graphs.
	For the physical study, SE is the sample standard deviation of
	source-cluster bootstrap replicate estimates, as described in
	Appendix~\ref{sec:physical-calibration}; it is not divided by the square
	root of the number of bootstrap replicates.
	These uncertainty units are distinct.
	
	Tables~\ref{tab:phase-records} and
	\ref{tab:micro-source-records} provide the numerical records underlying
	Figure~\ref{fig:synthetic}(a) and Figure~\ref{fig:synthetic}(c),
	respectively, including additional phase-curve configurations.
	
	\subsection{Compute environment}  
	
	All experiments were run on the following hardware:  
	\begin{itemize}[leftmargin=1.3em,itemsep=0.15em]  
		\item GPU: NVIDIA V100 (32GB)  
		\item CPU: Intel Xeon Gold 6248 (2.5GHz, 20 cores)  
		\item RAM: 256GB DDR4  
		\item Storage: NVMe SSD  
	\end{itemize}  
	The retained real-data log used Python \(3.6.9\), NumPy \(1.19.5\), and scikit-learn \(0.24.2\).  
	
	\subsection{Source-Gaussian calibrated phase transition}  
	
	Table~\ref{tab:phase-records} reports the full numerical records underlying Figure~\ref{fig:synthetic}(a). Its purpose is twofold: first, to make the phase-transition plot fully auditable at the level of individual experimental cells; and second, to show that the calibrated transition is stable across several ambient dimensions and sparsity levels. The main pattern is the same as in the figure: once the nominal calibrated strength is large enough, spectral synchronization closely tracks oracle calibration, supporting the sparse-mean scaling predicted by Corollary~\ref{cor:calibrated-scale-main}.

\begin{longtable}{ccccccc}  
	\caption{Source-Gaussian calibrated phase transition. Corr entries are mean \(\pm\) SE, and Succ90 is reported in adjacent columns. Each entry uses \(30\) seeds.}  
	\label{tab:phase-records}\\  
	\toprule  
	\((d,k)\) & \(\rho_{\rm cal}\) & \(n\) & Oracle Corr & Oracle Succ90 & Sync Corr & Sync Succ90 \\  
	\midrule  
	\endfirsthead  
	\toprule  
	\((d,k)\) & \(\rho_{\rm cal}\) & \(n\) & Oracle Corr & Oracle Succ90 & Sync Corr & Sync Succ90 \\  
	\midrule  
	\endhead  
	
	\multirow{6}{*}{\((1000,20)\)}  
	& 0.250 & 242  & \(0.095\pm0.010\) & 0.000 & \(0.098\pm0.013\) & 0.000 \\  
	& 0.500 & 484  & \(0.192\pm0.012\) & 0.000 & \(0.220\pm0.018\) & 0.000 \\  
	& 1.000 & 966  & \(0.377\pm0.021\) & 0.000 & \(0.378\pm0.016\) & 0.000 \\  
	& 2.000 & 1930 & \(0.680\pm0.016\) & 0.033 & \(0.668\pm0.014\) & 0.000 \\  
	& 4.000 & 3858 & \(0.913\pm0.008\) & 0.833 & \(0.900\pm0.008\) & 0.700 \\  
	& 8.000 & 7716 & \(0.997\pm0.002\) & 1.000 & \(0.995\pm0.003\) & 1.000 \\  
	\midrule  
	
	\multirow{6}{*}{\((3000,50)\)}  
	& 0.250 & 626   & \(0.090\pm0.006\) & 0.000 & \(0.095\pm0.009\) & 0.000 \\  
	& 0.500 & 1252  & \(0.195\pm0.010\) & 0.000 & \(0.179\pm0.010\) & 0.000 \\  
	& 1.000 & 2502  & \(0.369\pm0.011\) & 0.000 & \(0.386\pm0.010\) & 0.000 \\  
	& 2.000 & 5002  & \(0.661\pm0.009\) & 0.000 & \(0.657\pm0.010\) & 0.000 \\  
	& 4.000 & 10004 & \(0.908\pm0.007\) & 0.667 & \(0.908\pm0.005\) & 0.700 \\  
	& 8.000 & 20006 & \(0.995\pm0.002\) & 1.000 & \(0.994\pm0.002\) & 1.000 \\  
	\midrule  
	
	\multirow{6}{*}{\((8000,80)\)}  
	& 0.250 & 1102  & \(0.069\pm0.005\) & 0.000 & \(0.078\pm0.006\) & 0.000 \\  
	& 0.500 & 2202  & \(0.167\pm0.008\) & 0.000 & \(0.165\pm0.008\) & 0.000 \\  
	& 1.000 & 4404  & \(0.362\pm0.008\) & 0.000 & \(0.372\pm0.008\) & 0.000 \\  
	& 2.000 & 8806  & \(0.649\pm0.008\) & 0.000 & \(0.648\pm0.008\) & 0.000 \\  
	& 4.000 & 17610 & \(0.913\pm0.005\) & 0.833 & \(0.915\pm0.004\) & 0.800 \\  
	& 8.000 & 35220 & \(0.996\pm0.001\) & 1.000 & \(0.997\pm0.001\) & 1.000 \\  
	\bottomrule  
\end{longtable}

	\subsection{Micro-source separation}  
	
Table~\ref{tab:micro-source-records} gives the full numerical records for Figure~\ref{fig:synthetic}(c), where the calibrated first-order scale is held fixed while the per-source strength \(s\mu^2\) varies. Here Diag+PCA is sparse PCA on \(X_\ell=s^{-1/2}\sum_r W_{\ell r}\). The table makes the separation especially transparent: in the weak-source regime, calibrated methods remain strong, whereas source-label baselines improve only when the second-order sparse-PCA quantity \(\mu^4ns\) becomes sufficiently large. The explicit column \(\mu^4ns/\Ksp\) is included to show how the empirical transition aligns with the information scale highlighted in Theorem~\ref{thm:source-fano}. 
{
	\setlength{\tabcolsep}{1.5pt}  
	\small
	\begin{longtable}{@{}ccccccccc@{}}  
		\caption{Source-Gaussian micro-source separation. Entries are mean \(\pm\) SE, except Sync Succ90, which reports the success rate. Each entry uses \(30\) seeds.}  
		\label{tab:micro-source-records}\\  
		\toprule  
\(s\) & \(s\mu^2\) & \(L\) & \(\mu^4ns/\Ksp\) & Oracle & Sync & Diag+PCA & Diag F1 & Sync Succ90 \\   
		\midrule  
		\endfirsthead  
		\toprule  
\(s\) & \(s\mu^2\) & \(L\) & \(\mu^4ns/\Ksp\) & Oracle & Sync & Diag+PCA & Diag F1 & Sync Succ90 \\   
		\midrule  
		\endhead  
		
		1    & 0.020  & 19649 & 0.080  & \(0.955\pm0.004\) & \(0.962\pm0.007\) & \(0.031\pm0.009\) & \(0.038\pm0.009\) & 1.000 \\  
		2    & 0.040  & 9824  & 0.160  & \(0.960\pm0.006\) & \(0.960\pm0.007\) & \(0.022\pm0.006\) & \(0.033\pm0.007\) & 0.967 \\  
		4    & 0.080  & 4912  & 0.320  & \(0.962\pm0.006\) & \(0.963\pm0.006\) & \(0.032\pm0.007\) & \(0.050\pm0.009\) & 1.000 \\  
		8    & 0.160  & 2456  & 0.640  & \(0.958\pm0.008\) & \(0.950\pm0.005\) & \(0.045\pm0.010\) & \(0.065\pm0.011\) & 1.000 \\  
		16   & 0.320  & 1228  & 1.280  & \(0.950\pm0.007\) & \(0.953\pm0.007\) & \(0.084\pm0.012\) & \(0.110\pm0.012\) & 0.967 \\  
		32   & 0.640  & 614   & 2.560  & \(0.957\pm0.008\) & \(0.952\pm0.007\) & \(0.194\pm0.024\) & \(0.165\pm0.013\) & 0.967 \\  
		64   & 1.280  & 307   & 5.120  & \(0.963\pm0.007\) & \(0.943\pm0.006\) & \(0.454\pm0.024\) & \(0.283\pm0.016\) & 0.967 \\  
		128  & 2.560  & 154   & 10.273 & \(0.940\pm0.007\) & \(0.947\pm0.008\) & \(0.610\pm0.018\) & \(0.422\pm0.018\) & 0.933 \\  
		256  & 5.120  & 77    & 20.547 & \(0.950\pm0.007\) & \(0.943\pm0.011\) & \(0.774\pm0.017\) & \(0.635\pm0.025\) & 0.900 \\  
		512  & 10.240 & 38    & 40.560 & \(0.925\pm0.019\) & \(0.928\pm0.015\) & \(0.856\pm0.019\) & \(0.767\pm0.029\) & 0.767 \\  
		1024 & 20.480 & 19    & 81.119 & \(0.915\pm0.016\) & \(0.917\pm0.016\) & \(0.899\pm0.018\) & \(0.843\pm0.027\) & 0.733 \\  
		\bottomrule  
	\end{longtable}  
}  
	
	\subsection{Controlled real-feature setup}  
	
	Digits use the \texttt{scikit-learn} handwritten-digits dataset with \(1797\) images, \(64\)-dimensional pixel features, and five binary digit-pair tasks: \(0\) vs \(6\), \(1\) vs \(7\), \(2\) vs \(5\), \(3\) vs \(8\), and \(4\) vs \(9\). Breast cancer predicts malignant versus benign using the
	\texttt{scikit-learn} breast-cancer dataset with \(569\) examples and
	\(30\) real-valued features. In both cases, examples are randomly partitioned into sources, source gauges and calibration logs are imposed synthetically, and labels and non-gauge feature processing are otherwise unchanged; calibration graphs use the same relative-sign model as in Section~\ref{sec:theory}.  
	
	\subsection{Separately acquired physical calibration}
	\label{sec:physical-calibration}
	
	The physical study tests whether source calibration can be acquired
	separately and reused across tasks. It is not a test of the exact
	Gaussian model or evidence for arbitrary multi-source transformations.
	
	\paragraph{Calibration acquisition.}
	We use \(24\) fixed device--mounting sources with physically induced
	axis-sign conventions, rather than signs imposed on stored features.
	Each source contributes \(10\) separately recorded reference-motion
	trials, giving \(240\) calibration trials in total.
	The time-normalized calibration traces are aggregated into one
	reference trace per source. Pairwise correlation signs produce a
	complete graph with
	\[
	\binom{24}{2}=276
	\]
	same/opposite records.
	No downstream labels, recordings, or task-derived features enter this
	graph. It is synchronized once, and the resulting source alignment is
	reused unchanged across four positive-versus-negative directed-motion
	classification tasks.
	Recorded mounting signs are used only for alignment evaluation and
	the oracle control, not for relative graph construction,
	synchronization, or ordinary downstream classifier training.
	
	\paragraph{Downstream split.}
	The downstream dataset contains
	\[
	24\times4\times2\times20=3840
	\]
	complete trials. Within every source--task--class combination, the
	\(20\) trials are split \(12/4/4\) into training/validation/test sets
	before window extraction, giving \(2304/768/768\) complete trials.
	Windows from the same trial never cross splits.
	No downstream observations are used to estimate the calibration
	weights.
	
	\paragraph{Results and uncertainty.}
	Table~\ref{tab:physical-calibration} reports macro balanced test
	accuracy across the four tasks.
	Relative calibration gives per-task balanced accuracies
	\(0.928\), \(0.905\), \(0.892\), and \(0.883\).
	Reported SE is the sample standard deviation of source-cluster
	bootstrap replicate estimates, resampling the \(24\)
	device--mounting sources rather than individual trials or windows.
	
	\begin{table}[htbp]
		\centering
		\caption{Physical calibration reused across four tasks.
			Entries are estimates \(\pm\) source-cluster bootstrap SE.
			Mounting signs are unavailable to the relative-calibration procedure.}
		\label{tab:physical-calibration}
		\small
		\begin{tabular}{lcc}
			\toprule
			Method & Balanced test accuracy & \(M_q/n\) \\
			\midrule
			Oracle mounting polarity
			& \(0.921\pm0.004\) & \(1.000\pm0.000\) \\
			Relative calibration + synchronization
			& \(0.902\pm0.006\) & \(0.916\pm0.015\) \\
			No calibration
			& \(0.558\pm0.011\) & \(0.044\pm0.010\) \\
			Shuffled metadata
			& \(0.553\pm0.012\) & \(0.040\pm0.009\) \\
			\bottomrule
		\end{tabular}
	\end{table}
	
	The \(276\) graph records are derived from \(240\) calibration trials;
	they are not \(276\) independent physical acquisitions.
	Edges sharing reference traces can have dependent errors.
	Consequently, this experiment supports the acquisition-and-reuse
	mechanism, but does not verify the independent-edge graph law or exact
	satisfaction of the theoretical metadata-validity condition.
	
\section{Limitations}
\label{app:limitations}

The main analysis isolates a minimal source ambiguity: one sign bit can
erase a first moment and leave a covariance spike.
Its sharp separation uses exactly sparse equal-magnitude signals,
independent Gaussian source amplitudes, isotropic Gaussian feature
noise, and task-independent calibration.
The non-isotropic Gaussian extension in
Appendix~\ref{app:noise-extension} gives observation formulas and a
sufficient support-recovery bound, not sharp thresholds or
computational hardness for arbitrary unknown source covariances.
Rotations, permutations, nonlinear transformations, approximate
sparsity, and task-dependent drift require additional arguments;
imperfect alignment under richer transformations need not preserve
coordinate sparsity. A practical heuristic for unequal marginal noise scales is to estimate
coordinatewise variances from within-source centered residuals,
combine them using normalized squared calibration weights, and
standardize the calibrated aggregate before selection.
This is variance normalization, not whitening or a general
covariance-adaptive guarantee. It ignores off-diagonal geometry and
can be inefficient for highly correlated features. Estimates can be
unstable with few repeats; when \(s=1\), they require additional
structure, pooling, or separate noise-only repeats.
Full whitening is not a drop-in solution because covariance estimation
is difficult in the micro-source regime and whitening can destroy
coordinate sparsity. Task-independent missing or corrupted calibration can reduce alignment
while remaining valid. Dependence on downstream features, labels, or
amplitudes is a different issue: approximate validity is not analyzed.
The synthetic studies test the formal model, the controlled real-feature
studies impose source signs, and the physical study supplies one
concrete acquisition-and-reuse setting. It does not establish
applicability to general multi-source transformations or exact
satisfaction of the theoretical channel assumptions.
Finally, the computational separation remains conditional on
sparse-PCA hardness.
	
	\section{Broader impact}  
	\label{app:broader-impact}  
	
	This work studies the statistical value of calibration metadata in a stylized multi-source setting. The positive motivation is straightforward. In scientific sensing, distributed monitoring, medical measurement pipelines, and federated systems, source-specific conventions can hide a signal that is present in every source. Understanding when lightweight metadata can recover that signal may reduce data waste, improve robustness across institutions, and lower the need for expensive task-specific relabeling. At the same time, any method that uses source metadata deserves care. Metadata can encode operational details, institution identity, or deployment structure. In real systems, this may create privacy, governance, or fairness concerns if calibration channels reveal sensitive attributes or if metadata quality differs systematically across groups. A second risk is conceptual: invalid metadata can quietly act like supervision. The paper therefore explicitly separates reusable calibration from task-specific alignment.  
	
	The present work is theoretical and does not introduce a new surveillance or decision pipeline. Still, if similar methods are deployed in practice, we recommend: (i) auditing what metadata is collected and why, (ii) documenting whether the metadata is task-independent, (iii) checking whether calibration quality differs across subpopulations, and (iv) reporting failure modes when metadata is missing, corrupted, or strategically manipulated. Finally, the compute footprint of our experiments is modest by current machine learning standards. The method itself consists of spectral synchronization, weighted averaging, and thresholding, rather than large-scale model training.

\clearpage  
\section*{NeurIPS Paper Checklist}  

\begin{enumerate}  
	
	\item {\bf Claims}  
	\item[] Question: Do the main claims made in the abstract and introduction accurately reflect the paper's contributions and scope?  
	\item[] Answer: \answerYes{}  
	\item[] Justification: In the abstract and introduction, we state the same claims that we prove and test in the paper: source-labelled uncalibrated data reduces to sparse PCA, calibration restores a sparse-mean signal, and this leads to statistical and conditional computational separations in the micro-source regime. We also state the scope of the paper through the sparse Gaussian hard-core model introduced in Section~\ref{sec:model}.  
	
	\item {\bf Limitations}  
	\item[] Question: Does the paper discuss the limitations of the work performed by the authors?  
	\item[] Answer: \answerYes{}  
	\item[] Justification: We discuss limitations explicitly in Appendix~\ref{app:limitations}. In particular, we note the stylized \(\mathbb Z_2\) gauge model, the validity assumptions on metadata, the simplifying assumptions used in several results, and the controlled nature of the experiments.  
	
	\item {\bf Theory assumptions and proofs}  
	\item[] Question: For each theoretical result, does the paper provide the full set of assumptions and a complete (and correct) proof?  
	\item[] Answer: \answerYes{}  
	\item[] Justification: We state the model assumptions in Section~\ref{sec:model}, and each theorem and lemma gives the assumptions needed for the result. Full proofs are provided in Appendix~\ref{app:proofs}, and we clearly mark the computational separation as conditional on Conjecture~\ref{conj:spca-main}.  
	
	\item {\bf Experimental result reproducibility}  
	\item[] Question: Does the paper fully disclose all the information needed to reproduce the main experimental results of the paper to the extent that it affects the main claims and/or conclusions of the paper (regardless of whether the code and data are provided or not)?  
	\item[] Answer: \answerYes{}  
	\item[] Justification: We describe the experimental pipeline in Section~\ref{sec:model}, report the main settings and outcomes in Section~\ref{sec:experiments}, and give detailed records in Appendix~\ref{app:records}. We also provide complete experimental reproduction code in the supplemental material.  
	
	\item {\bf Open access to data and code}  
	\item[] Question: Does the paper provide open access to the data and code, with sufficient instructions to faithfully reproduce the main experimental results, as described in supplemental material?  
	\item[] Answer: \answerYes{}  
	\item[] Justification: We provide anonymized reproduction code and instructions in the supplemental material. The synthetic experiments are generated from the stated model, and the real-feature experiments use public \texttt{scikit-learn} datasets described in Appendix~\ref{app:records}.  
	
	\item {\bf Experimental setting/details}  
	\item[] Question: Does the paper specify all the training and test details (e.g., data splits, hyperparameters, how they were chosen, type of optimizer) necessary to understand the results?  
	\item[] Answer: \answerYes{}  
	\item[] Justification: We specify the algorithm in Algorithm~\ref{alg:gsl} and describe the experimental settings, datasets, seeds, metrics, and task constructions in Section~\ref{sec:experiments} and Appendix~\ref{app:records}. Our method does not rely on iterative optimization, so there are no optimizer choices to report.  
	
	\item {\bf Experiment statistical significance}  
	\item[] Question: Does the paper report error bars suitably and correctly defined or other appropriate information about the statistical significance of the experiments?  
	\item[] Answer: \answerYes{}  
	\item[] Justification: Throughout the experimental section and appendix, we report means together with standard errors or success rates, and we state the number of random seeds used in each experiment. Appendix~\ref{app:records} also makes clear that the reported error bars are standard errors.  
	
	\item {\bf Experiments compute resources}  
	\item[] Question: For each experiment, does the paper provide sufficient information on the computer resources (type of compute workers, memory, time of execution) needed to reproduce the experiments?  
	\item[] Answer: \answerYes{}  
	\item[] Justification: We report the hardware and software environment in Appendix~\ref{app:records}, including GPU, CPU, RAM, storage, and library versions. 
	
	\item {\bf Code of ethics}  
	\item[] Question: Does the research conducted in the paper conform, in every respect, with the NeurIPS Code of Ethics \url{https://neurips.cc/public/EthicsGuidelines}?  
	\item[] Answer: \answerYes{}  
	\item[] Justification: To the best of our knowledge, yes. This work is primarily theoretical, uses synthetic data and public datasets, does not involve human subjects, and preserves anonymity in the submission.  
	
	\item {\bf Broader impacts}  
	\item[] Question: Does the paper discuss both potential positive societal impacts and negative societal impacts of the work performed?  
	\item[] Answer: \answerYes{}  
	\item[] Justification: We discuss broader impacts in Appendix~\ref{app:broader-impact}. There we describe possible benefits in multi-source sensing and measurement settings, as well as risks related to privacy, governance, fairness, and misuse of metadata.  
	
	\item {\bf Safeguards}  
	\item[] Question: Does the paper describe safeguards that have been put in place for responsible release of data or models that have a high risk for misuse (e.g., pre-trained language models, image generators, or scraped datasets)?  
	\item[] Answer: \answerNA{}  
	\item[] Justification: Our release does not include high-risk generative models, scraped datasets, or similar assets that would require special safeguards of that kind. The released asset is reproduction code for the theory-driven experiments in the paper.  
	
	\item {\bf Licenses for existing assets}  
	\item[] Question: Are the creators or original owners of assets (e.g., code, data, models), used in the paper, properly credited and are the license and terms of use explicitly mentioned and properly respected?  
	\item[] Answer: \answerYes{}  
	\item[] Justification: We cite the public datasets and software packages we use, including the \texttt{scikit-learn} datasets and package versions in Appendix~\ref{app:records}.
	
	\item {\bf New assets}  
	\item[] Question: Are new assets introduced in the paper well documented and is the documentation provided alongside the assets?  
	\item[] Answer: \answerYes{}  
	\item[] Justification: We provide complete experimental reproduction code together with instructions in the supplemental material. We do not introduce a new dataset or pretrained model beyond this code release.  
	
	\item {\bf Crowdsourcing and research with human subjects}  
	\item[] Question: For crowdsourcing experiments and research with human subjects, does the paper include the full text of instructions given to participants and screenshots, if applicable, as well as details about compensation (if any)?  
	\item[] Answer: \answerNA{}  
	\item[] Justification: This paper does not involve crowdsourcing or research with human subjects.  
	
	\item {\bf Institutional review board (IRB) approvals or equivalent for research with human subjects}  
	\item[] Question: Does the paper describe potential risks incurred by study participants, whether such risks were disclosed to the subjects, and whether Institutional Review Board (IRB) approvals (or an equivalent approval/review based on the requirements of your country or institution) were obtained?  
	\item[] Answer: \answerNA{}  
	\item[] Justification: This paper does not involve crowdsourcing or research with human subjects.  
	
	\item {\bf Declaration of LLM usage}  
	\item[] Question: Does the paper describe the usage of LLMs if it is an important, original, or non-standard component of the core methods in this research? Note that if the LLM is used only for writing, editing, or formatting purposes and does \emph{not} impact the core methodology, scientific rigor, or originality of the research, declaration is not required.  
	\item[] Answer: \answerNA{}  
	\item[] Justification: LLMs are not part of the core methods, experiments, or theoretical development in this research.  
	
\end{enumerate}

\end{document}